\documentclass{article}
\usepackage{iclr2027_conference,times}
\iclrfinalcopy

\usepackage[utf8]{inputenc} %
\usepackage[T1]{fontenc}    %
\usepackage[hidelinks]{hyperref}      %
\usepackage{url}            %
\usepackage{booktabs}       %
\usepackage{amsfonts}       %

\usepackage{nicefrac}       %
\usepackage{microtype}      %
\usepackage[table]{xcolor}  %
\usepackage{amsmath}        %
\usepackage{amssymb}        %
\usepackage{algorithm}      %
\usepackage{algpseudocode}  %

\usepackage{graphicx}
\usepackage{wrapfig}
\usepackage{amsthm}   %

\usepackage{multirow}
\newtheorem{theorem}{Theorem}[section]
\newtheorem{proposition}[theorem]{Proposition}
\newtheorem{lemma}[theorem]{Lemma}
\newtheorem{corollary}[theorem]{Corollary}
\newtheorem{property}[theorem]{Property}
\newtheorem{definition}[theorem]{Definition}
\newtheorem{assumption}[theorem]{Assumption}

\newtheoremstyle{remarkitalic}%
  {3pt}%
  {3pt}%
  {\itshape}%
  {}%
  {\bfseries}%
  {.}%
  {0.5em}%
  {}%

\theoremstyle{remarkitalic}

\newtheorem*{remark*}{Remark}

\newcommand{\appsectionline}[2]{%
  \vspace{1.1em}%
  \noindent
  \makebox[2.2em][l]{\textbf{\hyperref[#1]{\ref*{#1}}}}%
  \textbf{\hyperref[#1]{#2}}%
  \hfill
  \textbf{\hyperref[#1]{\pageref*{#1}}}%
  \par
}

\newcommand{\appsubsectionline}[2]{%
  \noindent\hspace*{2.5em}%
  \makebox[3.5em][l]{\hyperref[#1]{\ref*{#1}}}%
  \hyperref[#1]{#2}%
  \nobreak\leaders\hbox to 0.6em{\hss.\hss}\hfill
  \hyperref[#1]{\pageref*{#1}}%
  \par
}

\usepackage{comment}
\title{Hierarchy-GBP: Accelerating Factor Graph\\
Inference via Abstraction and Recovery}

\author{%
\parbox{\dimexpr\textwidth-2\tabcolsep\relax}{%
\centering\normalfont
\begin{tabular}{@{}c@{}}
Yuzhou Cheng\textsuperscript{1} \quad Tom Yates\textsuperscript{1} \quad Ignacio Alzugaray\textsuperscript{1} \quad Danyal Akarca\textsuperscript{1} \\[0.2em]
Pedro A. M. Mediano\textsuperscript{1,\(\dagger\)} \quad Andrew J. Davison\textsuperscript{1,\(\dagger\)}
\end{tabular}%
}%
}

\usepackage{enumitem}

\begin{document}

\maketitle
\lhead{} %
\footnotetext[1]{Imperial College London, London, United Kingdom.}
\begingroup
\renewcommand{\thefootnote}{\fnsymbol{footnote}}
\footnotetext[2]{Co-supervised this work.}
\endgroup

\begin{abstract}
Gaussian Belief Propagation (GBP) is a distributed inference algorithm that passes messages in graphical models, making it attractive for scalable spatial intelligence. 
However, we find GBP most effective locally: it rapidly smooths message errors that vary sharply between neighbor variables, but corrects global errors across distant graph regions incrementally through long-range message propagations. 
We propose \emph{Hierarchy-GBP (H-GBP)}, an iterative, two-stage framework that accelerates GBP by first solving these global errors with a coarse graph approximation (abstraction) and projecting the results back to the original graph (recovery), then refining the remaining local errors with GBP. We prove H-GBP convergence to the optimum by deriving the combined matrix operator of our abstraction and recovery steps and analyzing its spectral radius.
Experiments on linear sparse graphs show that H-GBP converges fundamentally faster than standard GBP. Moreover, we validate H-GBP on two important spatial problems: Pose Graph Optimization (PGO) and Bundle Adjustment (BA). H-GBP markedly accelerates large-scale PGO and achieves state-of-the-art runtime across all tested BA scales.

\end{abstract}

\begin{figure}[!h]
    \centering
\includegraphics[width=\linewidth]{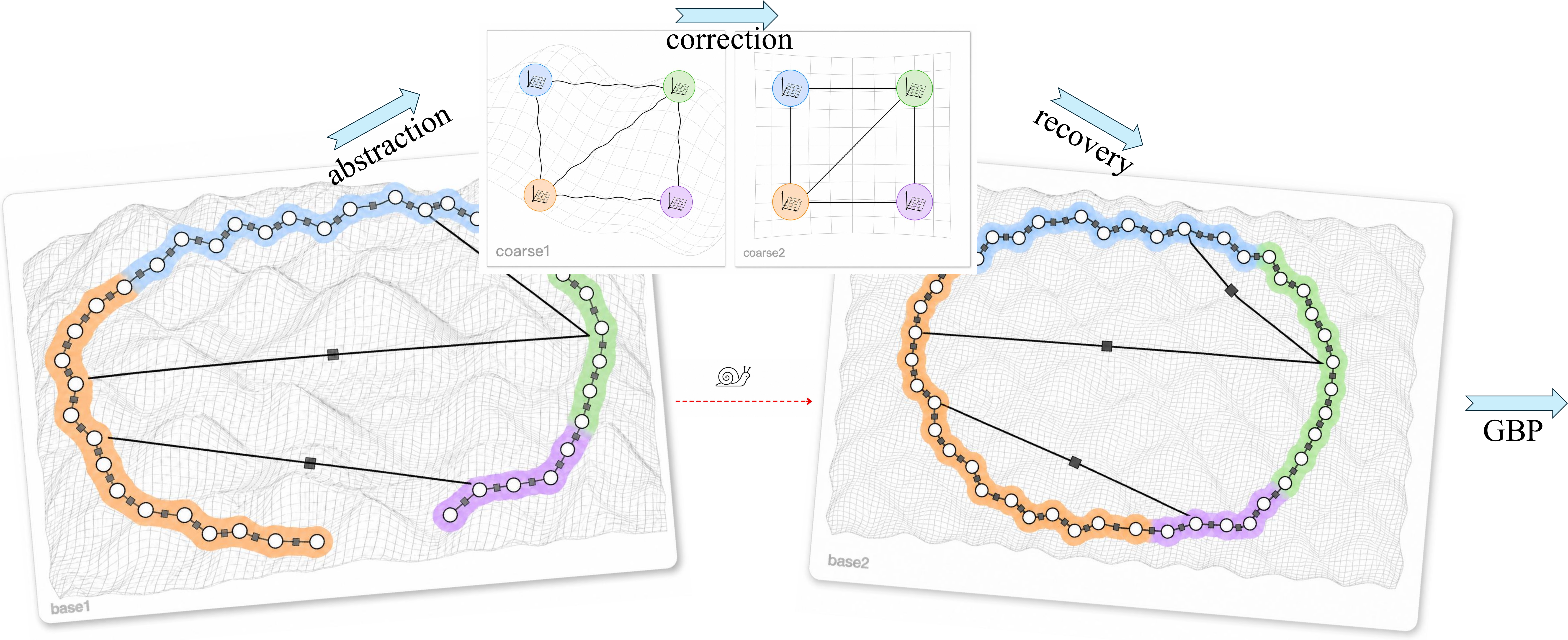}
        \vspace{-10pt}
\setlength{\abovecaptionskip}{1pt}
    \caption{Illustration of Hierarchy-GBP: A fine pose graph is partitioned along odometry order into local groups, where each group extracts a low-rank subspace for coarse abstraction. The hierarchical cycle first removes coherent coarse error and then lets GBP damp the remaining local error.}
    \label{fig:teaser_image}
\end{figure}

\vspace{-5pt}
\begin{figure}[t]
    \centering
    \includegraphics[width=\linewidth]{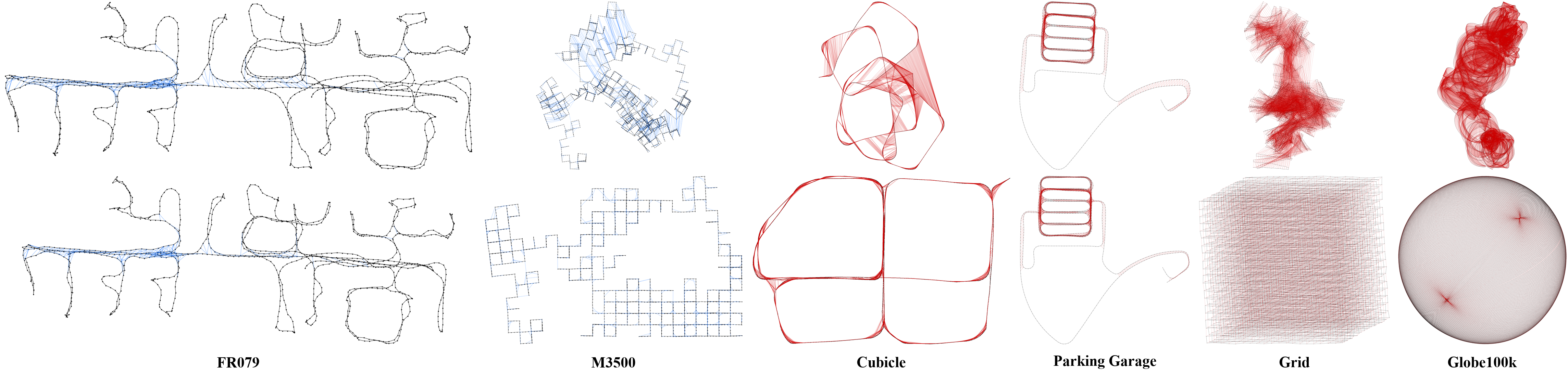}
        \vspace{-10pt}
    \caption{
    Pose graph's initial odometry trajectories versus trajectories after H-GBP optimization.
    }
    \label{fig:experiment_figure}
\end{figure}

\vspace{-5pt}
\begin{figure}[t]
    \centering
    \includegraphics[width=\linewidth]{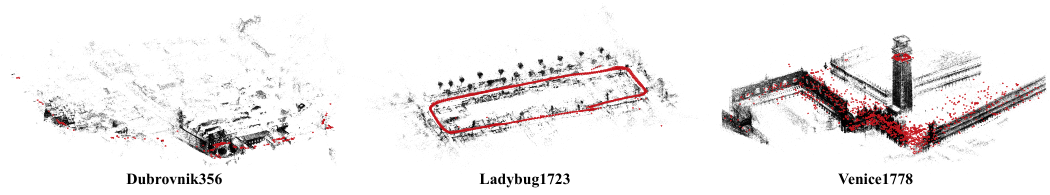}
        \vspace{-10pt}
    \caption{
    Bundle adjustment reconstructions with H-GBP optimization.
    }
\label{fig:experiment_figure2}
\end{figure}

\section{Introduction}
Gaussian Belief Propagation (GBP) is an efficient distributed graphical inference algorithm \citep{pearl2022reverend, yates2026belief} that has been successfully applied in error-correcting decoding \citep{frey1997revolution}, multiuser detection \citep{bickson2008multiuser}, and solving large-scale linear systems \citep{shental2008gaussian}. 
Inference proceeds by passing Gaussian distributions between nodes in a probabilistic graphical model, iteratively updating each variable's marginal belief.
This local, decoupled, and scalable computational model makes GBP a highly promising inference and reasoning framework for future spatial intelligence systems \citep{davison2019futuremapping}.

GBP has the desirable property that, if it converges on a Gaussian graph, it converges to the global optimal mean \citep{weiss1999correctness}. 
However, its convergence rate is compromised by graph structures like local loops or long graph diameters \citep{johnson2005walksum,malioutov2006walksums,du2018vectorvalued}, and GBP can often require many iterations \citep{elkurdi2012relaxed,liu2012feedback}. 
Algebraically, GBP behaves as a smoother: it quickly damps local (high-frequency) errors that vary sharply between neighboring nodes, while global (low-frequency) errors across distant graph regions require many iterations to reduce. 
We observe the local-versus-global error separation in many spatial inference problems, including SLAM, 3D computer vision, and sensor networks.

This motivates a two-level view of GBP (see Fig.\ref{fig:teaser_image}). The \emph{fine space} is the original variable space where GBP performs local message passing, while the \emph{coarse space} is a low-dimensional graph approximation suitable for removing low-frequency errors. We solidify this idea with a localized \emph{low-rank abstraction}. Variables are grouped locally and abstracted via their most prominent eigenvalues, forming a coarse problem that captures low-frequency errors \citep{briggs2000multigrid,nataf2011dtn,spillane2014geneo}. 

Our algorithm \emph{Hierarchy-GBP (H-GBP)} alternates coarse residual correction with fine-level synchronous GBP. 
For linear Gaussian factor graphs, we derive an operator describing a `V-cycle' of alternating coarse and fine steps, which lets us compute the spectral radius and predict a large improvement in convergence rate compared to standard GBP on the fine graph. Experiments on sparse spatial factor graphs \citep{dellaert2017factorgraphs, davison2018futuremapping,papadimitriou2025spatial,jiao2026opennavmap} verify this predicted acceleration in linear problems. We also show that the abstraction-and-recovery principle extends to practically important, non-linear Pose Graph Optimization (PGO) (see Fig.\ref{fig:experiment_figure}) and Bundle Adjustment (BA) (see Fig.\ref{fig:experiment_figure2}).

Specifically, our main contributions are:
\begin{itemize}

\item \textbf{H-GBP framework} We introduce a multilevel GBP framework for general sparse factor graphs. H-GBP builds a low-rank coarse space from low-eigenvalue modes, and alternates coarse corrections with GBP updates on the original fine graph.

\item \textbf{Empirical validation} H-GBP converges faster than standard GBP on sparse linear factor graphs empirically. Extensive experiments on PGO and BA benchmarks show H-GBP's speed and accuracy benefits, and demonstrate the potential for broader spatial intelligence applications.

\item \textbf{Convergence theory} For linear factor graphs, we establish a global precision convergence rate and derive a reduced cycle operator whose spectral radius is a geometric global convergence criterion. We prove the local convergence of H-GBP on manifolds.

\end{itemize}

\section{Related Work}

GBP on linear-Gaussian graphs has a strong theoretical foundation.
\citet{weiss1999correctness} shows that when loopy GBP converges on a Gaussian
graphical model, it recovers the correct posterior means, though the variances can be
overconfident. Subsequent work connected GBP to walk-sum expansions and
linear/quadratic solvers, clarifying that its behavior is governed by global coupling
structure rather than by any single local loop in isolation
\citep{malioutov2006walksums,shental2008gaussian, moallemi2009convergence}. More recent analyses extended
these perspectives to vector-valued Gaussian models, which arise in estimation problems \citep{du2018vectorvalued, yates2026belief}.
Our H-GBP analysis takes a \textbf{matrix-operator view} of the underlying dynamics.

A second line of work improves GBP efficiency while remaining on a flat graph. Representative
ideas include residual scheduling \citep{elidan2006residual}, relaxation or damping \citep{elkurdi2012relaxed}, and structural
modifications such as feedback message passing
\citep{liu2012feedback}. These methods can
reduce oscillation or avoid unnecessary updates, but they do so by
reordering or partially restructuring message passing on the original graph.
They do not build a coarse representation to divide and counter
long-range, low-frequency error components.

H-GBP is inspired by domain-decomposition methods, where local smoothers remove high-frequency errors and coarse spaces correct slowly varying components \citep{briggs2000multigrid}. The closest prior work combines GBP with multigrid for PDE/FEM systems \citep{elkurdi2014parallel,elkurdi2015parallel}. In contrast, H-GBP constructs the hierarchy directly on sparse factor graphs and derives the coupled coarse/fine message dynamics. To our knowledge, it is the first multilevel GBP framework formulated and analyzed for \textbf{general sparse factor graphs} beyond PDE/FEM discretizations.

Our work is also related to the factor-graph optimization literature in robotics \citep{tazaki2022spanning}. Factor graphs
have become a standard abstraction for large-scale estimation and SLAM, while sparse nonlinear
solvers such as g2o and Ceres remain the dominant practical baselines
\citep{dellaert2017factorgraphs,kuemmerle2011g2o, ceres}. Recently, there has
been renewed interest in GBP as a distributed inference engine for spatial estimation, including multi-robot localization, continuous-time SLAM, and bundle adjustment on graph processors
\citep{murai2024distributed,hug2024hyperion, ortiz2020bundle}. H-GBP
sits between these threads: it preserves the locality and parallelism of GBP with an abstract-then-recover acceleration mechanism. 

\section{Preliminaries}
\subsection{Factor Graphs} \label{sec:factor_graphs}
A factor graph $\mathcal G=(\mathcal V,\mathcal F,\mathcal E)$ is a bipartite graph with variable nodes $\mathcal V$, factor nodes $\mathcal F$ and edges $\mathcal E$ that represent the factorisation of a joint probability distribution.
\begin{equation}
p(x)\propto\prod_{f\in\mathcal F}\psi_f(x_{\mathcal N(f)})
\end{equation}
where $x=(x_1, ..., x_n)^T$ is the graph's state vector and $x_{\mathcal N(f)}$ is the subset of variables connected to factor $f$.

Variables are associated with the marginal posterior distribution (so-called belief) of their state $x_i$, and each factor $f \in \mathcal F$ defines a joint Gaussian potential $\psi_f$ over the states of its neighboring variables. 
Note that a factor's dimensionality depends on the number and dimension of its neighbours. 
In our work, all Gaussian potentials are defined in their canonical form, using the precision matrix $\Lambda_f$ and information vector $\eta_f$:
\[
\psi_f(x_{\mathcal N(f)})\propto
\exp\left\{
-\frac12 x_{\mathcal N(f)}^\top\Lambda_f x_{\mathcal N(f)}+\eta_f^\top x_{\mathcal N(f)}
\right\}
\]

Collecting these parameters from all factors then yields the global linear system $Ax=b$, where $A$ is the global precision matrix $\mathbf{\Lambda}$, $b$ is the global information vector $\boldsymbol{\eta}$, and $x$ is the global mean vector. 
The factor graph's unique Maximum A Posteriori (MAP) estimate is then given by $x_\star=A^{-1} b$.

\subsection{Gaussian Belief Propagation} \label{sec:GBP}

We provide here a brief overview of Gaussian Belief Propagation (GBP), and direct interested readers to \citet{ortiz2021visual} for further detail.

\begin{figure}[ht]
    \centering
    \includegraphics[width=\linewidth]{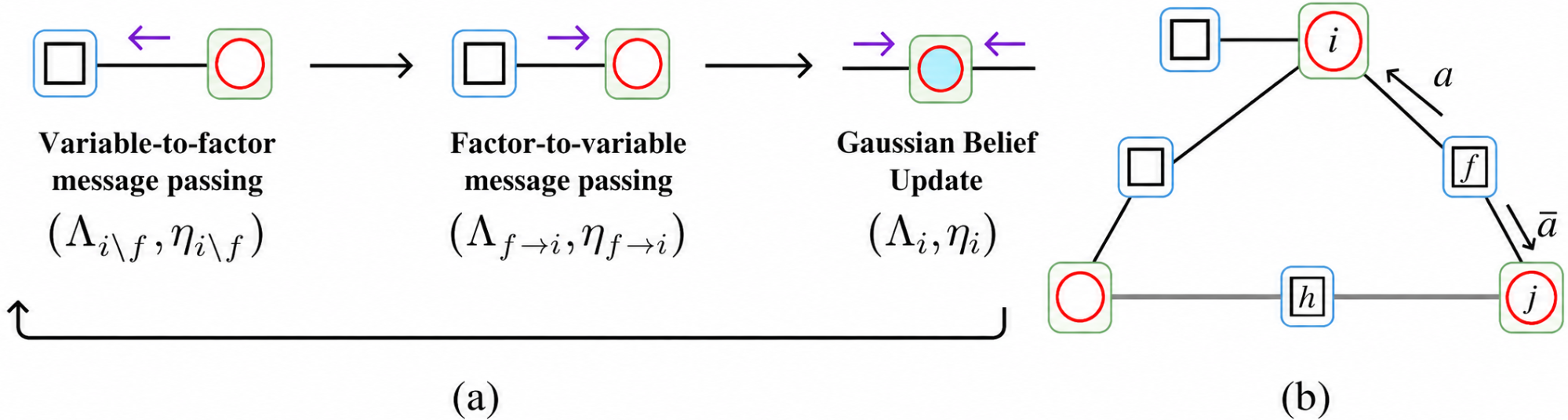}
        \vspace{-10pt}
        \setlength{\abovecaptionskip}{1.5pt}
    \caption{
    (a) Illustration of one synchronous iteration of Gaussian belief propagation. 
    (b) Local factor-variable indexing illustration for the formulation of Eqs. \eqref{eq:q-transfer} of message evolvement.}
    \label{fig:gbp_overview}
\end{figure}

Inference for GBP takes place through iterative message passing between neighboring nodes, where each message represents a marginal probability distribution over the state space of a single variable node.
Variable beliefs $(\widehat\Lambda_i, \widehat\eta_i)$ are then updated based on their most recently received messages (Fig. \ref{fig:gbp_overview} (a)). 
We detail GBP's 3 update equations below, showing only factors with degree 2 or less, which are applied iteratively for inference. 
The Schur complement handles higher-degree factor nodes in the same way, so all steps extend exactly.

\textbf{Variable-to-factor message update}:
Messages from variables to factors are the variable's belief with the most recent message from the receiving factor subtracted.
\begin{equation}
\label{eq:variable-to-factor}
\Lambda_{i\setminus f}=\widehat\Lambda_i-\Lambda_{f\to i} ,\qquad \eta_{i\setminus f}=\widehat\eta_i-\eta_{f\to i}.
\end{equation}

\textbf{Factor-to-variable message update}: 
Messages from factors to variables sum the messages from all other neighboring variables and marginalise over them to calculate the factor’s belief in terms of the receiving variable.
To do this, we block partition the factor's joint precision matrix over its neighboring variables $i$ and $j$:
\(
\Lambda_f
=
\begin{bmatrix}
\Lambda_{ii}^f & \Lambda_{ij}^f \\
\Lambda_{ji}^f & \Lambda_{jj}^f
\end{bmatrix}
\).
From here, we can calculate updated parameters for our factor-to-variable message, $\Lambda^+$ \& $\eta^+$:

\begin{equation}
\label{eq:schur-message}
\Lambda_{f\to i}^{+}=\Lambda_{ii}^f-\Lambda_{ij}^f(\Lambda_{jj}^f+\Lambda_{j\setminus f})^{-1}\Lambda_{ji}^f,\qquad
\eta_{f\to i}^{+}=\eta_i^f-\Lambda_{ij}^f(\Lambda_{jj}^f+\Lambda_{j\setminus f})^{-1}\bigl(\eta_j^f+\eta_{j\setminus f}\bigr)
\end{equation}

\textbf{Belief update}: 
Given an incoming message on each edge to variable $i$, where $i\in N(f)$, the belief in the variable is obtained by summing all incoming messages. 
\begin{equation}
\label{eq:variable-update}
\widehat\Lambda_i=\sum_{f\in N(i)}\Lambda_{f\to i},\qquad \widehat\eta_i=\sum_{f\in N(i)}\eta_{f\to i}, 
\end{equation}

In this work, we only consider synchronous GBP, in which one iteration consists of all variables in a graph updating their messages, all factors in the graph updating their messages, followed by all variables updating their belief. 

\section{Hierarchy-GBP: Algorithm and Matrix Operators}

Algorithm~\ref{alg:fresh-hgbp} summarizes H-GBP as an iterative linear solver for the Gaussian system \(Ax=b\). 
The prolongation matrix \(P\) is constructed once and kept fixed. 
Each cycle runs \(\boldsymbol{k}\) synchronous GBP sweeps, solves the coarse residual problem
\(P^\top A P\,\delta z=P^\top(b-Ax)\), and injects \(P\delta z\) using the message-preserving rule below \eqref{inconsistent_reconver}. The resulting two-level iteration combines global coarse correction with local message passing.

To better explain and analyze linear H-GBP, we derive the matrix operator for each step and combine them into the operator of \textbf{one full cycle}. 
The steps include:

\begin{enumerate}
    \item Abstraction \& Recovery of a coarse graph approximation
    \item The first GBP iteration after Recovery
    \item The remaining GBP iterations
\end{enumerate}

Note that separating steps 2 \& 3 is necessary because we only project variable beliefs, not graph messages, to the coarse graph approximation. 
Once this approximation is solved and the residuals recovered, the variable beliefs are updated but inconsistent with the non-updated graph messages. 
We show that this inconsistency occurs only for one GBP iteration, so we compute the resulting adjusted operator that represents the first iteration after abstraction and recovery. 

Let $x_\star=A^{-1}b$. We define the mean vector residual $v:=x-x_\star$. Concatenating all factor-to-variable information vectors yields a big information vector $\eta$. Let $\eta^\star$ be a fixed point of the H-GBP system. We define the information residual $m:=\eta-\eta^\star$. Note that $\eta^\star$ exists by construction (Appendix \ref{app:canonical-reference}).

\subsection{Operator 1: Abstraction \& Recovery of a Coarse Graph Approximation}
\label{sec:abs_rec_operator}

We derive the matrix operator for the abstraction and recovery step of H-GBP. 
To start, we group the variable nodes together, forming $G$ "supernodes".
For each supernode group $g$, we combine the precision matrices from factors wholly contained in the group into a single matrix $H_g$, where
\begin{equation}
\label{eq:fresh-local-operator}
(H_{g})_{i,j}=\sum_{ N(f)\subset g}(\Lambda_f)_{i,j}
\end{equation}

To reduce supernode dimensionality, for each group $g$ we perform eigenvalue decomposition on $H_g = U \Sigma U^\top$.
We then approximate $H_g$ with its \(\boldsymbol{r}\) smallest eigenvalues and corresponding eigenvectors $U_r$.

If we concatenate all groups' $U_{r}$ matrices blockwise into a graph-wide matrix $P$, we get a transformation from the original graph's space to an approximated, coarse space. 
This global prolongation matrix is assembled blockwise as
\[
\label{eq:fresh-prolongation}
P = 
\vspace{-0.5em}
\begin{bmatrix}
U_{r,1} & ... & \mathbf 0 \\
\vdots & \ddots & \vdots \\
\mathbf 0 & ... & U_{r,G} \\
\end{bmatrix}
\]

where $U_{r,g}$ represents $U_r$ for group $g$ and $\mathbf 0$ is the zero matrix.

The coarse correction solves 
\((P^\top AP)\delta z=P^\top(b-Ax)\). To inject this correction into the variable beliefs, we first add increment \(P\delta z\) to $x$. So, $x^+=x+P\delta z$. And then, we update each variable's information vector by this new mean \begin{equation}
\label{inconsistent_reconver}
\widehat\eta_i=\widehat\Lambda_i x_i^+.
\end{equation}

\begin{property}
   The mean error $v^+ $ after abstraction and recovery  is 
\begin{equation}
\label{eq:exact_coarse_correction}
v^+=C v, \quad C=I-P(P^\top A P)^{-1}P^\top A
\end{equation}
\end{property}

\subsection{Operator 2: The First GBP Iteration After Recovery}
\label{sec:first_gbp_operator}

To derive the operator for Synchronous GBP, we stack all information vector messages from factor-to-variable, which allows Eq.~\eqref{eq:schur-message} to be written in matrix form that represents the entire graph.
After the graph's precision matrices have converged, we can substitute Section \ref{sec:GBP}'s update equations (\eqref{eq:variable-update} into \eqref{eq:variable-to-factor} into \eqref{eq:schur-message}) and study the properties of the resulting matrix.

However, once the original graph's variable beliefs are updated from Eq. \eqref{inconsistent_reconver}, they are inconsistent with the graph's messages, which did not change under abstraction and recovery. 

As a result, the variable belief's information vector will not equal the sum of incoming information vectors. One synchronous GBP iteration will repair this, so we calculate the operator that represents the first iteration after abstraction and recovery in this section.

\begin{property}
After precision matrices have converged, we consider the errors in the variables' mean vector $v=x-x_\star$ and information vector $m=\eta-\eta^\star$. 
The first GBP iteration with inconsistent beliefs \& messages takes the form $m^+=Jv+\mathcal U m$, where for factor $f$ sending to variable $i$ with other neighboring variable $j$ (as in Fig \ref{fig:gbp_overview}b), these operators are given by:
\begin{equation}
\label{eq:first-sweep-repair}
    (Jv)_{f \to i}=-\Lambda_{i j}^f(\Lambda_{j j}^f+\widehat\Lambda_{j}^\star-\Lambda_{f \to j}^\star)^{-1}\widehat\Lambda_{j}^\star v_{j}, \qquad (\mathcal U m)_{f \to i}=\Lambda_{i j}^f(\Lambda_{j j}^f+\widehat\Lambda_{j}^\star-\Lambda_{f \to j}^\star)^{-1} m_{f \to j}.
\end{equation} 
\end{property}

$J$ and $\mathcal U$ are therefore our operators for the first iteration of GBP after abstraction and recovery.

\subsection{Operator 3: The Remaining Synchronous GBP iterations}
\label{sec:sync_gbp_operator}
We next derive the matrix operator for synchronous GBP iterations once Section \ref{sec:first_gbp_operator}'s inconsistency has been resolved.
As messages and variable beliefs are now consistent, we can substitute all update equations in Section \ref{sec:GBP} without trouble.

\begin{property}
Let 
\( \eta=\bigl[ 
\eta_{f \to i}^{\top}\;
\eta_{f \to j}^{\top}\;
\cdots\;
\bigr]^{\top} \)
denote the stacked column vector of all factor-to-variable information vector messages in graph $\mathcal G$, letting us write Eq.~\eqref{eq:schur-message} in matrix form.
When precisions have converged, we substitute Section \ref{sec:GBP}'s equations (\eqref{eq:variable-to-factor} into \eqref{eq:schur-message} into \eqref{eq:variable-update}) and show that the information vector's message update under one synchronous GBP iteration is of the affine form $\eta^+=Q\eta+s$, where the block entries of $Q$ are given by:
\vspace{-0.5em}
\begin{equation}
\label{eq:q-transfer}
Q_{(f \to i), (h \to j)} = -\Lambda_{i j}^f(\Lambda_{j j}^f+\widehat\Lambda_{j}^\star-\Lambda_{f \to j}^\star)^{-1}
\end{equation}
To demonstrate this property, a toy example is given in Appendix \ref{app:example_illustration}. Here, $j$ is the other neighboring node of $f$ than $i$ (see Fig. \ref{fig:gbp_overview} (b)). $Q_{(f \to i), (h \to j)}$ indexes $Q$ for the block corresponding to messages from factor $f$ to variable $i$, and factor $h$ to variable $j$, and the rest of the block entries of $Q$ are $\mathbf0$. 
\end{property}

The converged value $\eta^\star$ of GBP is unique and satisfies $\eta^\star=Q\eta^\star+s$. 
Recall that we define the error between our current eta-message and the fixed-point eta-message as $m=\eta-\eta^\star$. 
Changing variables from $\eta$ to $m$ removes the $s$ term, and we arrive at our matrix operator, $Q$, for one iteration of GBP:
\begin{equation}
\label{eq:m=Qm}
m^+=Qm.
\end{equation}
The message error after k-1 synchronous GBP iterations is then given by $m_{k+1}=Q^{k+1}m_0$, where $m_0$ is the initial mean error values.

Finally, H-GBP is iterative, and needs accurate variable means for the next abstraction and recovery. 
We therefore define a readout operator $\mathcal R$ that updates all errors in the variable means $v$ using the variables' information vector error values $m$. 

\vspace{-0.5em}
\begin{equation}
\label{readout}
v_i = (\mathcal Rm)_i=(\widehat\Lambda_i^\star)^{-1}\sum_{f \in N(i)}{m_{f\to i}}
\end{equation}
\vspace{-0.5em}

Combining all operators yields a single matrix-operator description of one full H-GBP cycle. 

\begin{property}
From operators $C$, $Jv + \mathcal U m$, the remaining $k-1$ iterations of operator $Q$ and $\mathcal R$, we get

\vspace{-0.5em}
\begin{equation}
\label{eq:taug}
\begin{bmatrix}
v^{+}\\
m^{+}
\end{bmatrix}
=
T_{{\rm aug},k}
\begin{bmatrix}
v\\
m
\end{bmatrix},
\qquad
T_{{\rm aug},k}
=
\begin{bmatrix}
C\mathcal R Q^{k-1}J
&
C\mathcal R Q^{k-1}\mathcal U 
\\[1mm]
Q^{k-1}J
&
Q^{k-1}\mathcal U
\end{bmatrix}.
\end{equation}
\vspace{-0.5em}

\end{property}

\subsection{Extension to manifolds}
\label{sec:manifold-hgbp}

H-GBP naturally extends to nonlinear factor graph optimization by operating on linearizations in the tangent space.  We use the standard
right-perturbation convention from Lie theory \citep{sola2018micro} and factor-graph
optimization~\citep{dellaert2017factorgraphs}: at the current
estimate \(\bar X\), each residual admits the first-order approximation
\(
r_f(\bar X_f\oplus\xi_f)\approx r_f(\bar X_f)+J_f\xi_f .
\)
This yields the linear Gaussian system in the tangent space that we solve using H-GBP, before retracting
\(\bar X_i\leftarrow \bar X_i\operatorname{Exp}(\xi_i)\). Optimization on manifolds involves repeating the above process continuously until \(\bar X_i\) converges.

\section{Empirical Results}
All experiments are on an Intel 13th Gen Core i9-13900HX CPU with 64 GB RAM, without GPU acceleration. Experimental details are provided in Appendix~\ref{app:additional}.

\subsection{Linear Hierarchy-GBP}

\begin{figure}[h]
    \centering
    \includegraphics[width=\linewidth]{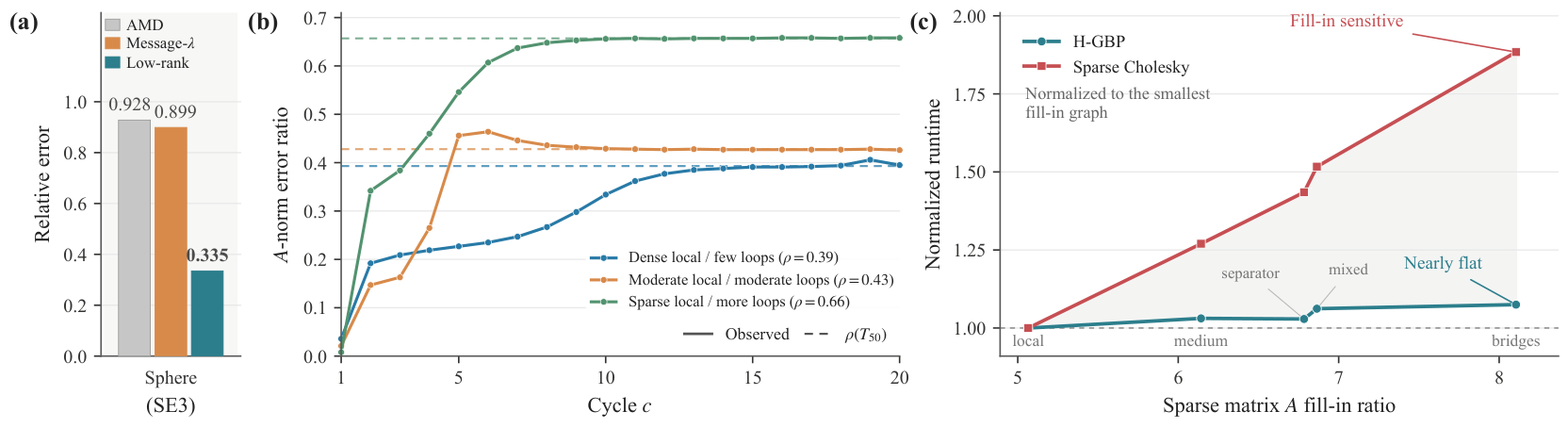}
    \setlength{\abovecaptionskip}{2pt}
    \caption{(a) Low-rank abstraction reduces error more effectively than alternative coarse spaces. (b) Successive-cycle error ratios approach the spectral-radius prediction as precision messages stabilize. (c) H-GBP runtime is less sensitive to sparse-Cholesky fill-in.}
    
    \vspace{-10pt}
    \label{fig:all}
\end{figure}

\begin{wrapfigure}{l}{0.49\textwidth}
    \vspace{-\baselineskip}
    \centering
    \includegraphics[width=\linewidth]{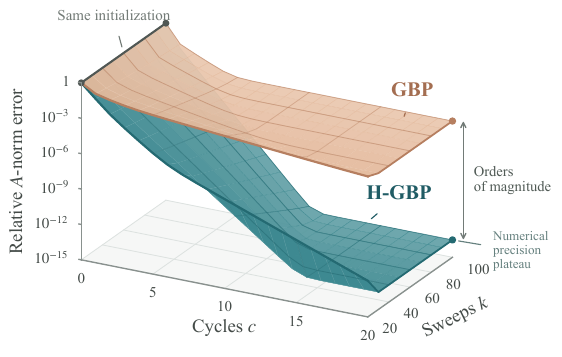}
    \setlength{\abovecaptionskip}{2pt}
    \caption{H-GBP versus GBP performance}
    \label{fig:linear_hgbp_vs_gbp}
    \vspace{-10pt}
\end{wrapfigure}

We study linear H-GBP on fixed Gaussian pose-graph systems, using \((x,y)\) variables for convergence comparisons and a fixed SE(3) linearization for the coarse-space ablation.
Figure~\ref{fig:linear_hgbp_vs_gbp} shows that H-GBP removes the low-frequency error that local message passing alone struggles to eliminate.
Varying both the number of cycles and the GBP iterations per cycle yields stable reduction, demonstrating H-GBP's consistent behavior across iteration budgets.

Figure~\ref{fig:all}(a) shows that local low-rank approximation yields substantially stronger error reduction than a basic multigrid-style abstraction under the same GBP budget, making our coarse-space choice important for sparse factor graphs. We find retaining four local modes per group gives a stable coarse correction with little sensitivity above this threshold. Figure~\ref{fig:all}(b) validates the correctness of the theory: as precision messages rapidly stabilize, the ratio of errors in successive cycles approaches \(\rho(T_{{\rm aug},50})\), matching the predicted asymptotic contraction factor.

A complementary experiment (Fig.~\ref{fig:all}(c)) evaluates solver performance across different levels of Cholesky fill-in. Sparse-Cholesky runtime grows \(1.88\times\), whereas H-GBP runtime grows only \(1.08\times\), revealing much weaker sensitivity to elimination fill-in.

\subsection{Hierarchy-GBP on \texorpdfstring{$\mathrm{SE}(2)$ and $\mathrm{SE}(3)$}{SE(2) and SE(3)} PGO}

\begin{table*}[h]
\centering
\caption{Pose-graph benchmarks: runtime (s) $\downarrow$ / final cost $\downarrow$.}
\label{tab:pgo_20260924}
\begingroup
\scriptsize
\setlength{\tabcolsep}{1.5pt}
\setlength{\fboxsep}{0.6pt}
\renewcommand{\arraystretch}{1.12}
\definecolor{benchfirst}{HTML}{FFCCCC}
\definecolor{benchsecond}{HTML}{FFE6C5}
\definecolor{benchthird}{HTML}{FFF8CF}
\newcommand{\pgoentry}[3]{%
  \ifcase#1\relax #2%
  \or\colorbox{benchfirst}{\textbf{#2}}%
  \or\colorbox{benchsecond}{#2}%
  \or\colorbox{benchthird}{#2}%
  \or\textcolor{red}{#2}%
  \fi\textcolor[HTML]{777777}{/}\textcolor[HTML]{444444}{#3}}
\resizebox{\textwidth}{!}{%
\begin{tabular}{@{}lclllllll@{}}
\toprule
& & \multicolumn{2}{c}{1 thread} & \multicolumn{5}{c}{16 threads} \\
\cmidrule(lr){3-4}\cmidrule(l){5-9}
\textbf{Dataset} & \textbf{Space}
& \shortstack[l]{\strut g2o\\\strut Cholesky\\{\tiny\bfseries\strut (ICRA 2011)}}
& \shortstack[l]{\strut\textbf{H-GBP}\\\strut\textbf{PGO}\\{\tiny\strut}}
& \shortstack[l]{\strut g2o\\\strut PCG\\{\tiny\bfseries\strut (ICRA 2011)}}
& \shortstack[l]{\strut Cycle\\\strut PGO\\{\tiny\bfseries\strut (T-RO 2021)}}
& \shortstack[l]{\strut AMM\\\strut PGO\\{\tiny\bfseries\strut (T-RO 2024)}}
& \shortstack[l]{\strut Schwarz\\\strut PGO\\{\tiny\bfseries\strut (FE Symposium 2026)}}
& \shortstack[l]{\strut\textbf{H-GBP}\\\strut\textbf{PGO}\\{\tiny\strut}} \\
\midrule
\textbf{FR079} & \(\mathrm{SE}(2)\)
  & \pgoentry{1}{0.02}{18.79}
  & \pgoentry{0}{0.29}{18.79}
  & \pgoentry{4}{1.09}{18.79}
  & \pgoentry{2}{0.04}{18.79}
  & \pgoentry{0}{0.17}{\textcolor{red}{129.13}}
  & \pgoentry{0}{0.14}{18.79}
  & \pgoentry{3}{0.11}{18.79} \\
\textbf{FRH} & \(\mathrm{SE}(2)\)
  & \pgoentry{1}{0.04}{9.66e-5}
  & \pgoentry{0}{0.35}{9.66e-5}
  & \pgoentry{0}{0.41}{9.66e-5}
  & \pgoentry{2}{0.09}{9.66e-5}
  & \pgoentry{0}{0.38}{1.02e-4}
  & \pgoentry{0}{0.13}{9.66e-5}
  & \pgoentry{3}{0.12}{9.66e-5} \\
\textbf{M3500} & \(\mathrm{SE}(2)\)
  & \pgoentry{1}{0.13}{68.96}
  & \pgoentry{0}{0.90}{68.96}
  & \pgoentry{0}{1.87}{68.96}
  & \pgoentry{2}{0.23}{68.96}
  & \pgoentry{0}{0.91}{70.45}
  & \pgoentry{0}{1.47}{68.96}
  & \pgoentry{3}{0.27}{68.96} \\
\textbf{ParkingGarage} & \(\mathrm{SE}(3)\)
  & \pgoentry{1}{0.38}{0.63}
  & \pgoentry{0}{7.56}{0.63}
  & \pgoentry{4}{19.72}{0.63}
  & \pgoentry{2}{0.85}{0.63}
  & \pgoentry{3}{1.37}{0.65}
  & \pgoentry{4}{49.78}{0.63}
  & \pgoentry{0}{1.77}{0.63} \\
\textbf{Sphere} & \(\mathrm{SE}(3)\)
  & \pgoentry{0}{2.02}{675.70}
  & \pgoentry{0}{4.76}{675.70}
  & \pgoentry{3}{1.35}{675.70}
  & \pgoentry{1}{1.08}{675.70}
  & \pgoentry{1}{1.08}{695.09}
  & \pgoentry{0}{2.72}{675.70}
  & \pgoentry{2}{1.18}{675.70} \\
\textbf{Cubicle} & \(\mathrm{SE}(3)\)
  & \pgoentry{2}{4.22}{1,378.53}
  & \pgoentry{0}{22.63}{1,378.53}
  & \pgoentry{4}{93.53}{1,378.53}
  & \pgoentry{1}{4.17}{1,609.14}
  & \pgoentry{0}{7.67}{\textcolor{red}{233,482.27}}
  & \pgoentry{0}{6.21}{1,378.53}
  & \pgoentry{3}{4.86}{1,378.53} \\
\textbf{Grid} & \(\mathrm{SE}(3)\)
  & \pgoentry{4}{493.90}{42,179.98}
  & \pgoentry{0}{16.16}{42,179.98}
  & \pgoentry{2}{4.38}{42,179.98}
  & \pgoentry{4}{621.95}{42,179.98}
  & \pgoentry{3}{7.94}{42,192.23}
  & \pgoentry{0}{44.55}{42,179.98}
  & \pgoentry{1}{3.90}{42,179.98} \\
\textbf{Globe10k} & \(\mathrm{SE}(3)\)
  & \pgoentry{0}{93.45}{3,039.93}
  & \pgoentry{0}{20.44}{3,039.93}
  & \pgoentry{2}{7.04}{3,039.93}
  & \pgoentry{0}{42.04}{3,039.93}
  & \pgoentry{3}{12.78}{3,040.00}
  & \pgoentry{0}{13.33}{3,039.93}
  & \pgoentry{1}{5.34}{3,039.93} \\
\textbf{Globe100k} & \(\mathrm{SE}(3)\)
  & \pgoentry{4}{1,288.89}{20,090.74}
  & \pgoentry{3}{177.88}{20,090.74}
  & \pgoentry{0}{662.81}{20,090.74}
  & \pgoentry{4}{OOM}{--}
  & \pgoentry{2}{132.07}{20,112.14}
  & \pgoentry{4}{\ensuremath{>2000}}{--}
  & \pgoentry{1}{42.71}{20,090.74} \\
\bottomrule
\end{tabular}%
}
\par\vspace{2pt}
{\fontsize{7}{8}\selectfont
\colorbox{benchfirst}{\textbf{1st}}\;
\colorbox{benchsecond}{2nd}\;
\colorbox{benchthird}{3rd}: fastest runtimes per dataset.\quad
\textcolor{red}{Red}: runtime $>20\times$ the fastest; cost $>2\times$ the lowest; out of memory (OOM).
\par}
\endgroup
\end{table*}

\begin{table*}[t]
\centering
\caption{Bundle adjustment benchmarks: runtime (s) $\downarrow$ / final sum-of-squares reprojection cost $\downarrow$.}
\label{tab:ba_20260924}
\begingroup
\scriptsize
\setlength{\tabcolsep}{1.5pt}
\setlength{\fboxsep}{0.6pt}
\renewcommand{\arraystretch}{1.12}
\definecolor{benchfirst}{HTML}{FFCCCC}
\definecolor{benchsecond}{HTML}{FFE6C5}
\definecolor{benchthird}{HTML}{FFF8CF}
\newcommand{\baentry}[3]{%
  \ifcase#1\relax #2%
  \or\colorbox{benchfirst}{\textbf{#2}}%
  \or\colorbox{benchsecond}{#2}%
  \or\colorbox{benchthird}{#2}%
  \or\textcolor{red}{#2}%
  \fi\textcolor[HTML]{777777}{/}\textcolor[HTML]{444444}{#3}}
\resizebox{\textwidth}{!}{%
\begin{tabular}{@{}llllllll@{}}
\toprule
& \multicolumn{2}{c}{1 thread} & \multicolumn{5}{c}{16 threads} \\
\cmidrule(lr){2-3}\cmidrule(l){4-8}
\textbf{Dataset}
& \shortstack[l]{\strut Ceres\\\strut Sparse\\{\tiny\bfseries\strut (Google Inc 2012)}}
& \shortstack[l]{\strut\textbf{H-GBP}\\\strut\textbf{BA}\\{\tiny\strut}}
& \shortstack[l]{\strut Ceres\\\strut Sparse\\{\tiny\bfseries\strut (Google Inc 2012)}}
& \shortstack[l]{\strut Ceres\\\strut Iterative\\{\tiny\bfseries\strut (Google Inc 2012)}}
& \shortstack[l]{\strut RootBA\\\strut QR\\{\tiny\bfseries\strut (CVPR 2021)}}
& \shortstack[l]{\strut PowerBA\\\strut\\{\tiny\bfseries\strut (CVPR 2023)}}
& \shortstack[l]{\strut\textbf{H-GBP}\\\strut\textbf{BA}\\{\tiny\strut}} \\
\midrule
\textbf{Ladybug49}
  & \baentry{0}{0.98}{26,690}
  & \baentry{2}{0.18}{26,705}
  & \baentry{0}{0.33}{26,690}
  & \baentry{0}{0.54}{26,691}
  & \baentry{0}{0.63}{26,691}
  & \baentry{3}{0.23}{27,299}
  & \baentry{1}{0.07}{26,705} \\
\textbf{Venice89}
  & \baentry{4}{19.22}{567,243}
  & \baentry{3}{2.68}{576,781}
  & \baentry{0}{5.91}{567,243}
  & \baentry{0}{5.03}{567,055}
  & \baentry{0}{6.22}{572,780}
  & \baentry{2}{1.91}{580,416}
  & \baentry{1}{0.59}{576,781} \\
\textbf{Final93}
  & \baentry{0}{7.39}{291,090}
  & \baentry{2}{1.59}{291,092}
  & \baentry{3}{1.69}{291,090}
  & \baentry{0}{4.29}{291,090}
  & \baentry{0}{6.84}{291,090}
  & \baentry{0}{1.79}{291,159}
  & \baentry{1}{0.38}{291,092} \\
\textbf{Trafalgar257}
  & \baentry{0}{7.21}{193,577}
  & \baentry{3}{1.31}{197,378}
  & \baentry{0}{2.38}{193,577}
  & \baentry{4}{14.41}{193,671}
  & \baentry{4}{22.72}{193,398}
  & \baentry{2}{1.11}{205,427}
  & \baentry{1}{0.45}{197,378} \\
\textbf{Dubrovnik356}
  & \baentry{4}{50.15}{987,859}
  & \baentry{3}{6.83}{1,021,102}
  & \baentry{0}{13.96}{987,859}
  & \baentry{4}{49.54}{987,859}
  & \baentry{4}{119.17}{987,866}
  & \baentry{2}{4.14}{1,059,446}
  & \baentry{1}{1.63}{1,021,102} \\
\textbf{Final394}
  & \baentry{4}{38.57}{596,506}
  & \baentry{3}{2.87}{598,296}
  & \baentry{0}{8.62}{596,506}
  & \baentry{0}{8.01}{596,536}
  & \baentry{4}{29.59}{596,578}
  & \baentry{2}{2.50}{611,101}
  & \baentry{1}{0.83}{598,296} \\
\textbf{Ladybug1723}
  & \baentry{4}{60.56}{778,095}
  & \baentry{3}{4.79}{759,600}
  & \baentry{0}{19.05}{778,095}
  & \baentry{0}{5.51}{770,444}
  & \baentry{0}{9.18}{763,070}
  & \baentry{2}{2.72}{797,704}
  & \baentry{1}{1.22}{759,599} \\
\textbf{Venice1778}
  & \baentry{4}{330.93}{3,277,387}
  & \baentry{3}{25.63}{3,363,149}
  & \baentry{0}{87.54}{3,277,387}
  & \baentry{0}{35.27}{\textcolor{red}{6,932,878}}
  & \baentry{0}{85.45}{3,369,978}
  & \baentry{2}{20.32}{3,391,089}
  & \baentry{1}{5.89}{3,363,149} \\
\textbf{Final3068}
  & \baentry{4}{586.17}{3,348,655}
  & \baentry{3}{12.40}{3,240,217}
  & \baentry{4}{142.78}{3,348,655}
  & \baentry{0}{13.17}{3,452,711}
  & \baentry{4}{102.23}{3,254,156}
  & \baentry{2}{6.15}{3,520,615}
  & \baentry{1}{3.03}{3,240,217} \\
\textbf{Final4585}
  & \baentry{4}{\ensuremath{>2000}}{--}
  & \baentry{3}{48.55}{10,331,432}
  & \baentry{4}{526.06}{10,783,988}
  & \baentry{0}{72.54}{13,021,324}
  & \baentry{4}{1,278.84}{10,158,450}
  & \baentry{2}{33.72}{10,441,192}
  & \baentry{1}{11.20}{10,331,432} \\
\bottomrule
\end{tabular}%
}
\par\vspace{2pt}
{\fontsize{7}{8}\selectfont
\colorbox{benchfirst}{\textbf{1st}}\;
\colorbox{benchsecond}{2nd}\;
\colorbox{benchthird}{3rd}: fastest runtimes per dataset.\quad
\textcolor{red}{Red}: runtime $>20\times$ the fastest; cost $>2\times$ the lowest.
\par}
\endgroup
\end{table*}

We evaluate H-GBP on \(\mathrm{SE}(2)\) and \(\mathrm{SE}(3)\) pose-graph benchmarks \citep{carlone2014fast, olson2006fast, carlone2015initialization, tazaki2022spanning} (Fig.~\ref{fig:experiment_figure}).
We compare against g2o-Cholesky and g2o-PCG~\citep{kuemmerle2011g2o}, AMM-PGO~\citep{fan2024mm_pgo}, Cycle-PGO~\citep{bai2021cycle}, and Schwarz-PGO~\citep{tee2026schwarz}.
Table~\ref{tab:pgo_20260924} reports runtime and final objective cost.
H-GBP groups every \(20\) consecutive nodes in odometry order, retaining the \(4\) smallest eigenmodes for \(\mathrm{SE}(2)\) and \(12\) for \(\mathrm{SE}(3)\).
Each relinearization is followed by \(3\) cycles of \(100\) GBP iterations for \(\mathrm{SE}(2)\), or \(5\) cycles of \(50\) GBP iterations for \(\mathrm{SE}(3)\).

H-GBP delivers consistent solution quality across all nine benchmarks: both its single-thread and \(16\)-thread final costs match g2o-Cholesky at the reported precision.
With \(16\) threads, it ranks among the three fastest methods on eight benchmarks and is the fastest on the three largest graphs, Grid, Globe10k, and Globe100k.
On these graphs, it achieves approximately \(\boldsymbol{126.6\times}\), \(\boldsymbol{17.5\times}\), and \(\boldsymbol{30.2\times}\) speedups over g2o-Cholesky, respectively.
The corresponding speedups over the fastest competing baseline on each graph are \(\boldsymbol{1.12\times}\), 
\(\boldsymbol{1.32\times}\), and \(\boldsymbol{3.09\times}\), showing that the large-scale advantage extends beyond comparison with a direct solver.

\vspace{-0.5em}

\subsection{Hierarchy-GBP on Bundle Adjustment}

We evaluate H-GBP BA on ten bundle-adjustment benchmarks against Ceres Solver~\citep{ceres} with either direct sparse linear solve or iterative PCG, RootBA-QR~\citep{demmel2021rootba}, and PowerBA~\citep{weber2023poba}.
Table~\ref{tab:ba_20260924} reports runtime and final sum-of-squares reprojection cost for the \(16\)-thread comparison, with single-thread Ceres-Sparse and H-GBP BA included as references.

H-GBP operates on the camera system \(S\Delta c=b\) obtained by Schur elimination of points at each linearization. For each pairwise potential, we simply fill in the corresponding off-diagonal blocks. The diagonal contributions define the unary potentials. (Note that this construction does not admit the assumptions. Therefore, the results of the BA are not intended to validate the theory, but rather to demonstrate the capabilities of H-GBP as an efficient solver.) Cameras are aggregated into connected groups of target size \(\max(2,\lceil n_c/24\rceil)\), and each group retains nine modes when abstracted.

H-GBP BA is the fastest on all ten datasets, achieving approximately \(\boldsymbol{2.0\times}\) \textbf{to} \(\boldsymbol{4.4\times}\) speedups over the fastest competing \(16\)-thread baseline on each dataset, and \(2.0\times\) to \(4.7\times\) over PowerBA.
This runtime advantage comes with consistently strong solution quality: H-GBP achieves lower final costs than PowerBA on every dataset, by up to \(8.0\%\), and remains within \(3.4\%\) of the lowest reported cost throughout the benchmark suite.
It also attains the lowest costs on Ladybug1723 and Final3068.
On the largest instance, Final4585, H-GBP finishes in \(11.20\)~s, approximately \(3.0\times\) faster than PowerBA and \(\boldsymbol{47.0\times}\) faster than \(16\)-thread Ceres-Sparse, while remaining within \(1.8\%\) of the best reported cost.
These results show that H-GBP sustains its \textbf{speed advantage} from small to large reconstructions \textbf{without} the large \textbf{cost outliers} observed for some baselines.

\section{Main theoretical results}
\label{theory_section}
We give the main results here, while all proof details are included in the Appendix \ref{app:theories}. We make the same assumptions as in vector-valued standard GBP literature \citep{du2018vectorvalued}. 

\begin{assumption}
    The factor-to-variable precision messages~\eqref{eq:schur-message} are initialized as positive semi-definite matrices:
\(
  \Lambda_{f\to i}^{(0)}\succeq 0
\quad\text{for all}~(f\to i).  
\)
\end{assumption}

\begin{assumption}
1. Each variable has a prior with positive definite precision.
2. Each binary factor precision admits
\(
\Lambda_f
=
\begin{bmatrix}
A_{f,i} & A_{f,j}
\end{bmatrix}^\top
W_f^{-1}
\begin{bmatrix}
A_{f,i} & A_{f,j}
\end{bmatrix},~
W_f\succ0,
\)
where \(A_{f,i}\) and \(A_{f,j}\) are full column rank.
\end{assumption}

A useful property of H-GBP, shared with standard GBP, is that the precision-message typically converges rapidly in practice, often well before the eta-message. The following theorem proves global geometric convergence of the precision, providing a tool that links dynamics from initialization to the limiting fixed point.

\begin{theorem} 
\label{global converge exp}
    There exist constants \(C_\Lambda<\infty\) and \(\theta_\Lambda\in(0,1)\),
depending on the initialization, such that
\[
\|\Lambda^{(\ell)}-\Lambda^\star\|\le C_\Lambda\theta_\Lambda^\ell,
\qquad \ell\ge0,
\]
for any fixed matrix norm after changing constants.
\end{theorem}

\vspace{-0.5em}
Let \(\mathcal M\) be the stacked directed \(\eta\)-message error space. At the end of every cycle, the coarse correction is applied to the readout~\eqref{readout} of the current message state while the messages are preserved. Hence, the cycle-boundary states satisfy
\(
v=C\mathcal Rm.
\)
Define the cycle-boundary subspace
\(
\mathcal Y_{\rm cyc}
:=
\{(C\mathcal Rm,m): m\in\mathcal M\}.
\)
Define the embedding and projection
\(
\iota_{\rm cyc}(m):=(C\mathcal Rm,m),
~
\pi_{\rm cyc}(v,m):=m.
\)
Then the exact reduced cycle operator on message space is
\begin{equation}
\label{eq:main-Mk}
M_k
:=
\pi_{\rm cyc}T_{{\rm aug},k}\iota_{\rm cyc}
=
Q^{k-1}(JC\mathcal R+\mathcal U).
\end{equation}

\begin{proposition}[Exact cycle-boundary reduction]
\label{prop:main-exact-reduction}
The subspace \(\mathcal Y_{\rm cyc}\) is \(T_{{\rm aug},k}\)-invariant. Moreover,
\[
T_{{\rm aug},k}\iota_{\rm cyc}
=
\iota_{\rm cyc}M_k.
\]
\end{proposition}

Thus \(M_k\) is the cycle-boundary reduction of the operator
\(T_{{\rm aug},k}\). Appendix~\ref{app:limiting-phase} shows that the reduced dynamics satisfy
\(
m_{t+1}=\widetilde M_{k,t}m_t=(M_k+E_t)m_t,~\text{where}~\sum_{t=0}^\infty\|E_t\|<\infty .
\)
Thus, the full H-GBP trajectory is a summably perturbed version of the limiting
linear cycle.

\begin{theorem}[Necessary and sufficient condition]
\label{thm:main-iff}
For the full H-GBP iteration, the following statements are equivalent:
\begin{enumerate}
\vspace{-0.5em}
\item for every admissible initialization, \(m_t\to0\);
\vspace{-0.5em}
\item for every compatible cycle-boundary state \(v_0=C\mathcal R_0m_0\),
      \((v_t,m_t)\to(0,0)\);
\vspace{-0.5em}
\item \(\rho(M_k)<1\);
\refstepcounter{enumi}\quad \theenumi. \(\rho(T_{{\rm aug},k})<1\).
\end{enumerate}
Whenever these conditions hold, for every fixed precision initialization and
every \(q\in(\rho(M_k),1)\), there exists \(C_q<\infty\) such that
\(
\|m_t\|\le C_q q^t\|m_0\|,~
\|v_t\|_A\le C_q q^t\|m_0\|.
\) The corresponding fine means converge to \(x_\star=A^{-1}b\).
\end{theorem}

\vspace{-0.5em}

The same linear result also provides a direct local convergence guarantee
for the nonlinear manifold extension in Section~4.4.

\begin{theorem}[Local convergence on manifolds]
Under the standard local conditions for inexact Gauss--Newton, suppose
that at every linearization the tangent H-GBP system satisfies
$\rho(M_{k,t})<1$. If each inner solve is run to a sufficiently small
fixed relative-residual tolerance, then the nonlinear H-GBP iterates
converge locally Q-linearly to a stationary point. 
\end{theorem}

\section{Conclusion}

We introduced Hierarchy-GBP, an abstraction-and-recovery framework to accelerate GBP on factor graphs by combining local message passing with a low-rank coarse correction.
We showed H-GBP's convergence by analyzing the spectral radius of a derived matrix operator for one cycle of linear H-GBP. 
We found H-GBP is competitive in both runtime and solution quality on challenging pose-graph and bundle adjustment benchmarks. 
The results highlight H-GBP as a promising design principle for scalable spatial-intelligence systems. 
Future work will explore broader sparse factor-graph problems with H-GBP, fully message-passing upper levels, and hierarchy-free distributed abstractions.

\newpage

\bibliographystyle{iclr2027_conference}
\bibliography{references}

@article{weiss1999correctness,
  title={Correctness of belief propagation in Gaussian graphical models of arbitrary topology},
  author={Weiss, Yair and Freeman, William},
  journal={Advances in neural information processing systems},
  volume={12},
  year={1999}
}

@article{davison2019futuremapping,
  title={FutureMapping 2: Gaussian belief propagation for spatial AI},
  author={Davison, Andrew J and Ortiz, Joseph},
  journal={arXiv preprint arXiv:1910.14139},
  year={2019}
}

@inproceedings{shental2008gaussian,
  title={Gaussian belief propagation solver for systems of linear equations},
  author={Shental, Ori and Siegel, Paul H and Wolf, Jack K and Bickson, Danny and Dolev, Danny},
  booktitle={2008 IEEE international symposium on information theory},
  pages={1863--1867},
  year={2008},
  organization={IEEE}
}

@article{johnson2005walksum,
  title={Walk-sum interpretation and analysis of Gaussian belief propagation},
  author={Malioutov, Dmitry and Willsky, Alan and Johnson, Jason},
  journal={Advances in neural information processing systems},
  volume={18},
  year={2005}
}

@article{malioutov2006walksums,
  title={Walk-sums and belief propagation in Gaussian graphical models},
  author={Malioutov, Dmitry M and Johnson, Jason K and Willsky, Alan S},
  journal={The Journal of Machine Learning Research},
  volume={7},
  pages={2031--2064},
  year={2006},
  publisher={JMLR. org}
}

@inproceedings{bickson2008multiuser,
  title={Gaussian belief propagation based multiuser detection},
  author={Bickson, Danny and Dolev, Danny and Shental, Ori and Siegel, Paul H and Wolf, Jack K},
  booktitle={2008 IEEE International Symposium on Information Theory},
  pages={1878--1882},
  year={2008},
  organization={IEEE}
}

@article{du2018vectorvalued,
  title={Convergence analysis of distributed inference with vector-valued Gaussian belief propagation},
  author={Du, Jian and Ma, Shaodan and Wu, Yik-Chung and Kar, Soummya and Moura, Jos{\'e} MF},
  journal={Journal of Machine Learning Research},
  volume={18},
  number={172},
  pages={1--38},
  year={2018}
}

@inproceedings{elidan2006residual,
  title={Residual belief propagation: Informed scheduling for asynchronous message passing},
  author={Elidan, Gal and McGraw, Ian and Koller, Daphne},
  booktitle={Proceedings of the Twenty-Second Conference on Uncertainty in Artificial Intelligence},
  pages={165--173},
  year={2006}
}

@inproceedings{elkurdi2012relaxed,
  title={Relaxed Gaussian belief propagation},
  author={El-Kurdi, Yousef and Giannacopoulos, Dennis and Gross, Warren J},
  booktitle={2012 IEEE International Symposium on Information Theory Proceedings},
  pages={2002--2006},
  year={2012},
  organization={IEEE}
}

@article{liu2012feedback,
  title={Feedback message passing for inference in Gaussian graphical models},
  author={Liu, Ying and Chandrasekaran, Venkat and Anandkumar, Animashree and Willsky, Alan S},
  journal={IEEE Transactions on Signal Processing},
  volume={60},
  number={8},
  pages={4135--4150},
  year={2012},
  publisher={IEEE}
}

@book{briggs2000multigrid,
  title={A multigrid tutorial},
  author={Briggs, William L and Henson, Van Emden and McCormick, Steve F},
  year={2000},
  publisher={SIAM}
}

@article{xu2017algebraic,
  title={Algebraic multigrid methods},
  author={Xu, Jinchao and Zikatanov, Ludmil},
  journal={Acta Numerica},
  volume={26},
  pages={591--721},
  year={2017},
  publisher={Cambridge University Press}
}

@article{nataf2011dtn,
  title={A coarse space construction based on local Dirichlet-to-Neumann maps},
  author={Nataf, Fr{\'e}d{\'e}ric and Xiang, Hua and Dolean, Victorita and Spillane, Nicole},
  journal={SIAM Journal on Scientific Computing},
  volume={33},
  number={4},
  pages={1623--1642},
  year={2011},
  publisher={SIAM}
}

@article{spillane2014geneo,
  title={Abstract robust coarse spaces for systems of PDEs via generalized eigenproblems in the overlaps},
  author={Spillane, Nicole and Dolean, Victorita and Hauret, Patrice and Nataf, Fr{\'e}d{\'e}ric and Pechstein, Clemens and Scheichl, Robert},
  journal={Numerische Mathematik},
  volume={126},
  number={4},
  pages={741--770},
  year={2014},
  publisher={Springer}
}

@article{elkurdi2014parallel,
  title={Parallel multigrid acceleration for the finite-element Gaussian belief propagation algorithm},
  author={El-Kurdi, Yousef and Gross, Warren J and Giannacopoulos, Dennis},
  journal={IEEE transactions on magnetics},
  volume={50},
  number={2},
  pages={581--584},
  year={2014},
  publisher={IEEE}
}

@article{yates2026belief,
  title={Belief Propagation Converges to Gaussian Distributions in Sparsely-Connected Factor Graphs},
  author={Yates, Tom and Cheng, Yuzhou and Alzugaray, Ignacio and Akarca, Danyal and Mediano, Pedro AM and Davison, Andrew J},
  journal={arXiv preprint arXiv:2601.21935},
  year={2026}
}

@article{elkurdi2015parallel,
  title={Parallel finite element technique using Gaussian belief propagation},
  author={El-Kurdi, Yousef and Dehnavi, Maryam Mehri and Gross, Warren J and Giannacopoulos, Dennis},
  journal={Computer Physics Communications},
  volume={193},
  pages={38--48},
  year={2015},
  publisher={Elsevier}
}

@software{ceres,
  author  = {Agarwal, Sameer and Mierle, Keir and The Ceres Solver Team},
  title   = {Ceres Solver},
  version = {2.2},
  year    = {2023},
  month   = {10},
  url     = {https://github.com/ceres-solver/ceres-solver}
}

@article{dellaert2017factorgraphs,
  title={Factor graphs for robot perception},
  author={Dellaert, Frank and Kaess, Michael},
  journal={Foundations and Trends{\textregistered} in Robotics},
  volume={6},
  number={1-2},
  pages={1--139},
  year={2017},
  publisher={Emerald Publishing Limited}
}

@article{sola2018micro,
  author       = {Joan Sol{\`a} and J{\'e}r{\'e}mie Deray and Dinesh Atchuthan},
  title        = {A Micro Lie Theory for State Estimation in Robotics},
  journal      = {arXiv preprint arXiv:1812.01537},
  year         = {2018},
  eprint       = {1812.01537},
  archivePrefix= {arXiv},
  primaryClass = {cs.RO}
}

@inproceedings{kuemmerle2011g2o,
  title={g 2 o: A general framework for graph optimization},
  author={K{\"u}mmerle, Rainer and Grisetti, Giorgio and Strasdat, Hauke and Konolige, Kurt and Burgard, Wolfram},
  booktitle={2011 IEEE international conference on robotics and automation},
  pages={3607--3613},
  year={2011},
  organization={IEEE}
}

@article{ortiz2021visual,
  title={A visual introduction to Gaussian belief propagation},
  author={Ortiz, Joseph and Evans, Talfan and Davison, Andrew J},
  journal={arXiv preprint arXiv:2107.02308},
  year={2021}
}

@inproceedings{ortiz2020bundle,
  title={Bundle adjustment on a graph processor},
  author={Ortiz, Joseph and Pupilli, Mark and Leutenegger, Stefan and Davison, Andrew J},
  booktitle={2020 IEEE/CVF Conference on Computer Vision and Pattern Recognition (CVPR)},
  pages={2413--2422},
  year={2020},
  organization={IEEE}
}

@article{murai2024distributed,
  title={Distributed simultaneous localisation and auto-calibration using gaussian belief propagation},
  author={Murai, Riku and Alzugaray, Ignacio and Kelly, Paul HJ and Davison, Andrew J},
  journal={IEEE Robotics and Automation Letters},
  volume={9},
  number={3},
  pages={2136--2143},
  year={2024},
  publisher={IEEE}
}

@inproceedings{hug2024hyperion,
  title={Hyperion--a fast, versatile symbolic gaussian belief propagation framework for continuous-time slam},
  author={Hug, David and Alzugaray, Ignacio and Chli, Margarita},
  booktitle={European Conference on Computer Vision},
  pages={215--231},
  year={2024},
  organization={Springer}
}

@incollection{pearl2022reverend,
  title={Reverend Bayes on inference engines: A distributed hierarchical approach},
  author={Pearl, Judea},
  booktitle={Probabilistic and causal inference: the works of Judea Pearl},
  pages={129--138},
  year={2022}
}

@article{frey1997revolution,
  title={A revolution: Belief propagation in graphs with cycles},
  author={Frey, Brendan J and MacKay, David},
  journal={Advances in neural information processing systems},
  volume={10},
  year={1997}
}

@article{tazaki2022spanning,
  title={A spanning tree-based multi-resolution approach for pose-graph optimization},
  author={Tazaki, Yuichi},
  journal={IEEE Robotics and Automation Letters},
  volume={7},
  number={4},
  pages={10033--10040},
  year={2022},
  publisher={IEEE}
}

@article{carlone2014fast,
  title={A fast and accurate approximation for planar pose graph optimization},
  author={Carlone, Luca and Aragues, Rosario and Castellanos, Jos{\'e} A and Bona, Basilio},
  journal={The International Journal of Robotics Research},
  volume={33},
  number={7},
  pages={965--987},
  year={2014},
  publisher={SAGE Publications Sage UK: London, England}
}

@inproceedings{olson2006fast,
  title={Fast iterative alignment of pose graphs with poor initial estimates},
  author={Olson, Edwin and Leonard, John and Teller, Seth},
  booktitle={Proceedings 2006 IEEE International Conference on Robotics and Automation, 2006. ICRA 2006.},
  pages={2262--2269},
  year={2006},
  organization={IEEE}
}

@inproceedings{carlone2015initialization,
  title={Initialization techniques for 3D SLAM: A survey on rotation estimation and its use in pose graph optimization},
  author={Carlone, Luca and Tron, Roberto and Daniilidis, Kostas and Dellaert, Frank},
  booktitle={2015 IEEE international conference on robotics and automation (ICRA)},
  pages={4597--4604},
  year={2015},
  organization={IEEE}
}

@book{higham2002accuracy,
  title={Accuracy and stability of numerical algorithms},
  author={Higham, Nicholas J},
  year={2002},
  publisher={SIAM}
}

@article{jiao2026opennavmap,
  title={OpenNavMap: Structure-Free Topometric Mapping via Large-Scale Collaborative Localization},
  author={Jiao, Jianhao and Liu, Changkun and Yu, Jingwen and Liu, Boyi and Zhang, Qianyi and Wang, Yue and Kanoulas, Dimitrios},
  journal={arXiv preprint arXiv:2601.12291},
  year={2026}
}

@article{davison2018futuremapping,
  title={FutureMapping: The computational structure of spatial AI systems},
  author={Davison, Andrew J},
  journal={arXiv preprint arXiv:1803.11288},
  year={2018}
}

@book{papadimitriou2025spatial,
  title={Spatial Artificial Intelligence},
  author={Papadimitriou, Fivos},
  year={2025},
  publisher={Springer}
}

@article{gratton2007approximate,
  title={Approximate {Gauss--Newton} Methods for Nonlinear Least Squares Problems},
  author={Gratton, S. and Lawless, A. S. and Nichols, N. K.},
  journal={SIAM Journal on Optimization},
  volume={18},
  number={1},
  pages={106--132},
  year={2007},
  doi={10.1137/050624935}
}

@article{dembo1982inexact,
  title={Inexact {Newton} Methods},
  author={Dembo, Ron S. and Eisenstat, Stanley C. and Steihaug, Trond},
  journal={SIAM Journal on Numerical Analysis},
  volume={19},
  number={2},
  pages={400--408},
  year={1982},
  doi={10.1137/0719025}
}

@book{ortega1970iterative,
  title={Iterative Solution of Nonlinear Equations in Several Variables},
  author={Ortega, J. M. and Rheinboldt, W. C.},
  publisher={Academic Press},
  year={1970}
}

@article{moallemi2009convergence,
  title={Convergence of min-sum message passing for quadratic optimization},
  author={Moallemi, Ciamac C and Van Roy, Benjamin},
  journal={IEEE Transactions on Information Theory},
  volume={55},
  number={5},
  pages={2413--2423},
  year={2009},
  publisher={IEEE}
}

@article{fan2024mm_pgo,
  title={Majorization Minimization Methods for Distributed Pose Graph Optimization},
  author={Fan, Taosha and Murphey, Todd D.},
  journal={IEEE Transactions on Robotics},
  volume={40},
  pages={22--42},
  year={2024}
}

@article{bai2021cycle,
  title={Sparse Pose Graph Optimization in Cycle Space},
  author={Bai, Fang and Vidal-Calleja, Teresa and Grisetti, Giorgio},
  journal={IEEE Transactions on Robotics},
  year={2021},
  doi={10.1109/TRO.2021.3050328}
}

@misc{tee2026schwarz,
  title={Overlapping Schwarz Preconditioners for Pose Graph {SLAM} in Robotics},
  author={Tee, Yue Xiang and K{\"o}hler, Stephan and Rheinbach, Oliver and Zug, Sebastian},
  howpublished={Chemnitz FE Symposium},
  year={2026},
  note={Conference presentation abstract},
  url={https://www.chemnitz-am.de/cfem2026/talks/tee_xiang_yue.pdf}
}

@inproceedings{demmel2021rootba,
  title={Square Root Bundle Adjustment for Large-Scale Reconstruction},
  author={Demmel, Nikolaus and Sommer, Christiane and Cremers, Daniel and Usenko, Vladyslav},
  booktitle={Proceedings of the IEEE/CVF Conference on Computer Vision and Pattern Recognition (CVPR)},
  year={2021}
}

@inproceedings{weber2023poba,
  title={Power Bundle Adjustment for Large-Scale {3D} Reconstruction},
  author={Weber, Simon and Demmel, Nikolaus and Chan, Tin Chon and Cremers, Daniel},
  booktitle={Proceedings of the IEEE/CVF Conference on Computer Vision and Pattern Recognition (CVPR)},
  pages={281--289},
  year={2023}
}

\clearpage
\appendix

\thispagestyle{empty}

\noindent{\LARGE\bfseries Supplementary Material}
\vspace{1.5em}

{\large\bfseries Contents}

\vspace{1em}

\appsectionline{app:example}{Examples}
\appsubsectionline{app:example_illustration}{A three-node four-factor example}
\appsubsectionline{app:gbp_failure}{Standard GBP failure case}

\appsectionline{app:theories}{Theory}
\appsubsectionline{app:precision-imported}{Convergence of precisions}
\appsubsectionline{app:canonical-reference}{MAP-centered message reference}
\appsubsectionline{app:linear-hgbp-clear}{Frozen linear H-GBP cycle}
\appsubsectionline{app:limiting-phase}{From the initial precision to the limiting cycle}
\appsubsectionline{app:necessary-sufficient}{Global convergence criterion and MAP target}
\appsubsectionline{app:hgbp-manifolds}{H-GBP on manifolds}

\appsectionline{app:algdetail}{Algorithm}

\appsectionline{app:additional}{Additional Details}
\appsubsectionline{app:impl}{Implementation Details}
\appsubsectionline{app:exp-details}{Linear Experiment Details}
\appsubsectionline{app:exp-details2}{SE2/SE3 Pose-Graph Experiment Details}
\appsubsectionline{app:exp-details-ba}{Bundle Adjustment Experiment Details}

\appsectionline{app:discussion}{Discussion}

\clearpage

\section{Examples}
\label{app:example}

\subsection{A three-node four-factor example}
\label{app:example_illustration}

We give a small scalar example (see Fig. \ref{fig:scalar_q_example}) illustrating how the limiting eta-transfer
matrix \(Q\) is assembled from the local Schur-complement coefficients.
Let \(x_0,x_1,x_2\in\mathbb R\). The graph has three binary factors
\(f_{01},f_{12},f_{02}\), and one unary factor \(f_0\) attached to
\(x_1\). 
\begin{figure}[ht]
    \centering
    \includegraphics[width=\linewidth]{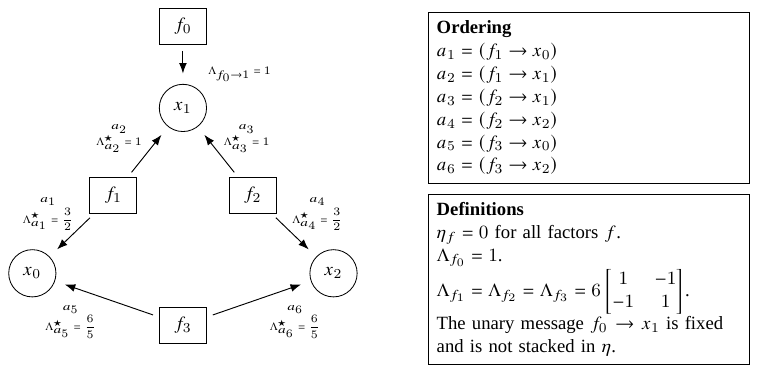}
    \vspace{-10pt}
    \caption{
    Scalar three-variable, four-factor example used to illustrate the construction of the limiting eta-message transfer matrix \(Q\). Circles denote variables, boxes denote factors, and arrows denote directed factor-to-variable messages.
}\label{fig:scalar_q_example}
\end{figure}

We set all information vectors to zero, since the affine term $s$ is irrelevant for the eta recursion. The unary precision is set as
\[
\Lambda_{f_0}=1,
\]
and each binary factor has precision
\[
\Lambda_{f_{1}}
= \Lambda_{f_{2}} = \Lambda_{f_{3}} =
6
\begin{bmatrix}
1 & -1\\
-1 & 1
\end{bmatrix}.
\]
Thus \(\Lambda_{ii}^f=6\), \(\Lambda_{ij}^f=-6\), and
\(\Lambda_{jj}^f=6\). The corresponding global information matrix is
\[
A=
\begin{bmatrix}
12 & -6 & -6\\
-6 & 13 & -6\\
-6 & -6 & 12
\end{bmatrix}.
\]

Order the directed binary messages as
\[
\begin{aligned}
a_1&=(f_{1}\to0),&
a_2&=(f_{1}\to1),&
a_3&=(f_{2}\to1),\\
a_4&=(f_{2}\to2),&
a_5&=(f_{3}\to0),&
a_6&=(f_{3}\to2).
\end{aligned}
\]
The unary factor sends a fixed message to \(x_1\) and is not included in
the stacked binary eta-message vector. Starting from zero binary factor-to-variable precision messages, i.e.,
\(\Lambda_{a_\ell}^{(0)}=0\) for all \(\ell=1,\dots,6\), the synchronous
precision update converges to a fixed point. The following table
lists the iterates.
\[
\begin{array}{c|cccccc}
t
& \Lambda_{a_1}^{(t)}
& \Lambda_{a_2}^{(t)}
& \Lambda_{a_3}^{(t)}
& \Lambda_{a_4}^{(t)}
& \Lambda_{a_5}^{(t)}
& \Lambda_{a_6}^{(t)}
\\
\hline
0  & 0.000 & 0.000 & 0.000 & 0.000 & 0.000 & 0.000\\
1  & 0.857 & 0.000 & 0.000 & 0.857 & 0.000 & 0.000\\
2  & 0.857 & 0.000 & 0.000 & 0.857 & 0.750 & 0.750\\
3  & 0.857 & 0.667 & 0.667 & 0.857 & 0.750 & 0.750\\
4  & 1.304 & 0.667 & 0.667 & 1.304 & 0.750 & 0.750\\
5  & 1.304 & 0.667 & 0.667 & 1.304 & 1.071 & 1.071\\
6  & 1.304 & 0.909 & 0.909 & 1.304 & 1.071 & 1.071\\
7  & 1.448 & 0.909 & 0.909 & 1.448 & 1.071 & 1.071\\
8  & 1.448 & 0.909 & 0.909 & 1.448 & 1.167 & 1.167\\
9  & 1.448 & 0.977 & 0.977 & 1.448 & 1.167 & 1.167\\
10 & 1.487 & 0.977 & 0.977 & 1.487 & 1.167 & 1.167\\
\vdots & \vdots & \vdots & \vdots & \vdots & \vdots & \vdots\\
19 & 1.500 & 0.999 & 0.999 & 1.500 & 1.199 & 1.199\\
20 & 1.500 & 1.000 & 1.000 & 1.500 & 1.200 & 1.200\\
21   & 1.500 & 1.000 & 1.000 & 1.500 & 1.200 & 1.200
\end{array}
\]
The exact fixed point for precision is therefore
\[
\Lambda_{a_1}^\star=\frac32,\quad
\Lambda_{a_2}^\star=1,\quad
\Lambda_{a_3}^\star=1,\quad
\Lambda_{a_4}^\star=\frac32,\quad
\Lambda_{a_5}^\star=\frac65,\quad
\Lambda_{a_6}^\star=\frac65.
\] 

Hence the limiting variable belief precisions are
\[
\widehat\Lambda_0^\star=\frac{27}{10},\qquad
\widehat\Lambda_1^\star=3,\qquad
\widehat\Lambda_2^\star=\frac{27}{10}.
\]

For a directed message \(a=(f\to i)\), let \(j(a)\) be the other endpoint
of \(f\), and let \(\bar a=(f\to j(a))\). Then
\[
D_a^\star
=
\Lambda_{j(a)j(a)}^f
+
\widehat\Lambda_{j(a)}^\star
-
\Lambda_{\bar a}^\star,
\qquad
B_a^\star
=
\Lambda_{i j(a)}^f(D_a^\star)^{-1}.
\]
For example,
\[
D_{a_1}^\star
=
6+\widehat\Lambda_1^\star-\Lambda_{a_2}^\star
=
6+3-1
=
8,
\qquad
B_{a_1}^\star
=
(-6)8^{-1}
=
-\frac34.
\]
The six Schur coefficients are
\[
\begin{array}{c|c|c|c}
a & D_a^\star & B_a^\star & -B_a^\star\\
\hline
a_1 & 8 & -\frac34 & \frac34\\
a_2 & \frac{36}{5} & -\frac56 & \frac56\\
a_3 & \frac{36}{5} & -\frac56 & \frac56\\
a_4 & 8 & -\frac34 & \frac34\\
a_5 & \frac{15}{2} & -\frac45 & \frac45\\
a_6 & \frac{15}{2} & -\frac45 & \frac45
\end{array}.
\]

Using the ordering \((a_1,\dots,a_6)\), the limiting eta-transfer matrix
is
\[
Q
=
\begin{bmatrix}
0 & 0 & \frac34 & 0 & 0 & 0\\
0 & 0 & 0 & 0 & \frac56 & 0\\
0 & 0 & 0 & 0 & 0 & \frac56\\
0 & \frac34 & 0 & 0 & 0 & 0\\
0 & 0 & 0 & \frac45 & 0 & 0\\
\frac45 & 0 & 0 & 0 & 0 & 0
\end{bmatrix}.
\]
For instance, the first row encodes
\(\eta_{a_1}^{+}=(3/4)\eta_{a_3}\): when sending \(f_{01}\to0\), the
factor looks through node \(1\), and the only binary incoming message at
node \(1\) excluding \(f_{1}\) is \(a_3=(f_{2}\to1)\). The unary factor
affects the denominator \(D_{a_1}^\star\) and the affine offset, but not
the linear message-to-message entry of \(Q\).

Finally, \(Q\) consists of two directed three-cycles, each with product
\[
\frac34\cdot\frac56\cdot\frac45=\frac12.
\]
Therefore
\[
\rho(Q)=\left(\frac12\right)^{1/3}\approx0.7937.
\]

\subsection{Standard GBP failure case}
\label{app:gbp_failure}
In Fig. \ref{fig:posegraph_example_annotated}, we construct an example with four variables (each two-dimensional) and multiple factors. Numerically, the dynamical index $\rho(Q)$ of the standard GBP is greater than 1. Hierarchy-GBP uses the low-rank operator, group \{0,1\} and \{2,3\} as two groups. By choosing coarse rank $r=3$ per group and performing $k=8$ GBP iterations per cycle, the Hierarchy-GBP index $\rho(M_8)$ is less than 1.

\[
\rho(Q)\approx 1.0019966,\qquad \rho(M_8)\approx 0.9205045.
\]
\begin{figure}[ht]
    \centering
    \includegraphics[width=\linewidth]{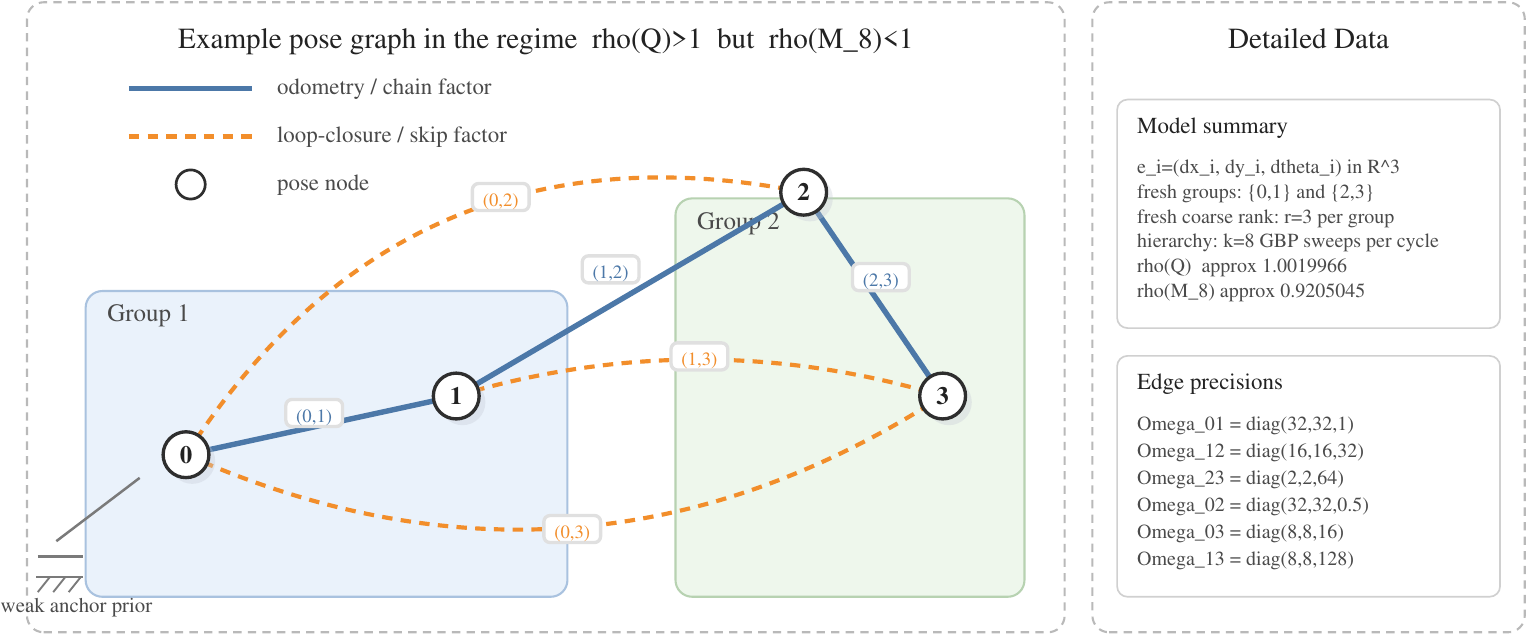}
        \vspace{-10pt}
    \caption{
    A four-node pose graph example, for which standard GBP does not converge, while H-GBP with $k=8$ sweeps per cycle will converge.}\label{fig:posegraph_example_annotated}
\end{figure}

\clearpage
\section{Theory}
\label{app:theories}

The purpose of this appendix is to prove the two convergence claims in the
main paper: first, for a fixed H-GBP cycle length \(k\), the full hierarchy
converges from every admissible initialization if and only if the reduced cycle
operator satisfies \(\rho(M_k)<1\) (equivalently \(\rho(T_{{\rm aug},k})<1\)).
Second, whenever this criterion holds, the fine-level mean converges to the exact MAP solution of the original linear Gaussian factor graph.

For readability, formulas are written for unary and binary factors.  Higher-order
factors use the same Schur-complement formulas \citep{ortiz2021visual} after replacing the single
``other endpoint'' by the block of all non-target variables.

We make the same assumptions (Assumption~\ref{ass:app-psd-init}, \ref{fair_assumption}) as in vector-valued standard GBP~\citep{du2018vectorvalued} and state their results as facts here, before delving into our results of H-GBP in the following subsections. 

Let \(N_1(i)\) and \(N_2(i)\) denote the unary and binary factors incident on
variable \(i\).  For a binary factor \(f=\{i,j\}\), write
\[
\Lambda_f=
\begin{bmatrix}
\Lambda_{ii}^f & \Lambda_{ij}^f\\
\Lambda_{ji}^f & \Lambda_{jj}^f
\end{bmatrix}.
\]
For a directed binary message \(a=(f\to i)\), with the other endpoint
\(j=j(a)\), define the cavity precision
\[
\mathcal C_a(\Lambda):=\Lambda_j^{\rm self}+S_a(\Lambda),
\qquad
\Lambda_j^{\rm self}:=\sum_{u\in N_1(j)}\Lambda_{u\to j},
\qquad
S_a(\Lambda):=
\sum_{h\in N_2(j)\setminus\{f\}}\Lambda_{h\to j}.
\]

\begin{assumption}
\label{ass:app-psd-init}
The initial factor-to-variable precision messages satisfy
\(
\Lambda_{f\to i}^{(0)}\succeq0
\)
for every directed factor-to-variable edge.
\end{assumption}

\begin{assumption}
\label{fair_assumption}
The following regularity conditions hold.
\begin{enumerate}[label=(\roman*)]
\item Each variable has positive self-precision:
\(
\Lambda_i^{\rm self}\succ0
\)
for all \(i\in\mathcal V\).
\item Each binary factor admits the measurement-form factorization
\[
\Lambda_f
=
\begin{bmatrix}A_{f,i}&A_{f,j}\end{bmatrix}^\top
W_f^{-1}
\begin{bmatrix}A_{f,i}&A_{f,j}\end{bmatrix},
\qquad
W_f\succ0,
\]
where \(A_{f,i}\) and \(A_{f,j}\) are full column rank.
\end{enumerate}
\end{assumption}

\begin{remark*}
Assumption \ref{fair_assumption} is consistent with standard Gaussian-BP implementations. In practice, the Schur-complement block in each factor-to-variable update is always stabilized
by adding a small positive diagonal term, for example
\[
D_{f\to i}\mapsto D_{f\to i}+\varepsilon I,
\qquad
\varepsilon\approx 10^{-10}.
\]
For the precision recursion, this is equivalent to adding an arbitrarily weak
unary regularizer with precision \(\varepsilon I\) to each variable. After
including this regularizer in \(N_1(i)\), we have
\(
\Lambda_i^{\rm self}
\succeq
\varepsilon I
\succ0 .
\)
Thus the condition \(\Lambda_i^{\rm self}\succ0\) should be understood as a
standard technical regularity condition ensuring well-defined Schur complements,
rather than as a restrictive modeling requirement.
\end{remark*}

Let $S_+$ and $S_{++}$ denote the blockwise positive semidefinite and positive definite cones. The precision update map is $F:S_+\to S_{++}$ such that
\begin{equation}
\label{eq:app-precision-map}
F_a(\Lambda)
=
\Lambda_{ii}^f-
\Lambda_{ij}^f
\bigl(\Lambda_{jj}^f+\mathcal C_a(\Lambda)\bigr)^{-1}
\Lambda_{ji}^f .
\end{equation}

Using the measurement factorization and the Woodbury identity, we get this fact from \citet{du2018vectorvalued}
\begin{equation}
\label{eq:app-measurement-form}
F_a(\Lambda)
=
A_{f,i}^\top
\left[
W_f+A_{f,j}\mathcal C_a(\Lambda)^{-1}A_{f,j}^\top
\right]^{-1}
A_{f,i}.
\end{equation}
This form makes positivity, monotonicity, and differentiability of $F$ transparent.

\begin{theorem}[standard GBP precision convergence \citep{du2018vectorvalued}]
\label{thm:app-precision-imported}
Under Assumptions~\ref{ass:app-psd-init}--\ref{fair_assumption}, the recursion
\(\Lambda^{(\ell+1)}=F(\Lambda^{(\ell)})\) has a unique positive-definite fixed
point \(\Lambda^\star\).  For every positive-semidefinite initialization,
\(\Lambda^{(\ell)}\to\Lambda^\star\).  Moreover, for every fixed neighborhood of
\(\Lambda^\star\), the trajectory approaches that neighborhood at an exponential
rate before the first entrance.
\end{theorem}

\clearpage
\subsection{Convergence of precisions}
\label{app:precision-imported}

\begin{proposition}[Precision dynamics are unchanged by the hierarchy]
\label{prop:app-hierarchy-precision}
The H-GBP coarse correction changes only the variable belief information vectors
and means.  It does not alter factor-to-variable precision messages.  Therefore
H-GBP has exactly the same precision recursion,
\(
\Lambda^{(\ell+1)}=F(\Lambda^{(\ell)})
\),
as standard synchronous GBP.
\end{proposition}

\begin{proof}
The message-preserving injection rewrites \(\widehat\eta_i\) so that the belief
mean is shifted by the coarse correction, while leaving
\(\widehat\Lambda_i\), \(\Lambda_{f\to i}\), and \(\eta_{f\to i}\) unchanged.
The next precision update therefore sees the same precision state as
ordinary GBP.
\end{proof}

Now, we strengthen the convergence result to a global geometric convergence rate, which will be recursively used to control the early, nonstationary precision phase of H-GBP.

\begin{lemma}[Derivative of the precision map at the fixed point]
\label{lem:app-DF-formula}
The precision map \(F\) is Fr\'echet differentiable on a neighborhood of
\(\Lambda^\star\) inside \(\mathbb S_{++}\).  For a directed message
\(a=(f\to i)\), let \(j=j(a)\), and define
\[
\mathcal B_a(\Lambda):=
W_f+A_{f,j}\mathcal C_a(\Lambda)^{-1}A_{f,j}^\top .
\]
Then for every symmetric blockwise increment \(H\),
\[
DF_a(\Lambda)[H]
=
A_{f,i}^\top\mathcal B_a(\Lambda)^{-1}
A_{f,j}\mathcal C_a(\Lambda)^{-1}
\left(\sum_{h\in N_2(j)\setminus\{f\}}H_{h\to j}\right)
\mathcal C_a(\Lambda)^{-1}
A_{f,j}^\top\mathcal B_a(\Lambda)^{-1}A_{f,i}.
\]
In particular, \(H\succeq0\) implies \(DF(\Lambda^\star)[H]\succeq0\).
\end{lemma}

\begin{proof}
Equation~\eqref{eq:app-measurement-form} gives
\(
F_a(\Lambda)
=
A_{f,i}^\top\mathcal B_a(\Lambda)^{-1}A_{f,i}.
\)

The map
\(
\mathcal C_a(\Lambda)
\)
is linear in the relevant components of \(\Lambda\). Hence, 
\(
\label{DC_a}
  D\mathcal C_a(\Lambda)[H]
=\sum_{h\in N_2(j)\setminus\{f\}}H_{h\to j}.  
\) Since \(\mathcal C_a(\Lambda^\star)\succ0\), by continuity
\(\mathcal C_a(\Lambda)\succ0\) in a sufficiently small neighborhood of
\(\Lambda^\star\). Thus the inverse is differentiable there. 

The Fr\'echet derivative formula for inversion~\citep{higham2002accuracy} is \(D(X^{-1})[H]=-X^{-1}HX^{-1}\), so
\[
D(\mathcal C_a^{-1})(\Lambda)[H]
=
-\mathcal C_a(\Lambda)^{-1}
D\mathcal C_a(\Lambda)[H]
\mathcal C_a(\Lambda)^{-1}.
\]

Next,
\(
\mathcal B_a(\Lambda)
=
W_f
+
A_{f,j}\mathcal C_a(\Lambda)^{-1}A_{f,j}^\top .
\)
Since \(W_f\), \(A_{f,j}\), and \(A_{f,j}^\top\) are fixed with respect to
\(\Lambda\), differentiating gives
\(
D\mathcal B_a(\Lambda)[H]
=
A_{f,j}
D(\mathcal C_a^{-1})(\Lambda)[H]
A_{f,j}^\top .
\)

Applying the same inversion derivative formula to
\(\mathcal B_a(\Lambda)^{-1}\), we have
\[
D(\mathcal B_a^{-1})(\Lambda)[H]
=
-\mathcal B_a(\Lambda)^{-1}
D\mathcal B_a(\Lambda)[H]
\mathcal B_a(\Lambda)^{-1}.
\]

Finally, since \(A_{f,i}\) is fixed with respect to \(\Lambda\),
\(
DF_a(\Lambda)[H]
=
A_{f,i}^\top
D(\mathcal B_a^{-1})(\Lambda)[H]
A_{f,i}.
\)
Hence, substituting all the equations above sequentially we get
\[
DF_a(\Lambda)[H]
=
A_{f,i}^\top
\mathcal B_a(\Lambda)^{-1}
A_{f,j}\mathcal C_a(\Lambda)^{-1}
\left(
\sum_{h\in N_2(j)\setminus\{f\}}H_{h\to j}
\right)
\mathcal C_a(\Lambda)^{-1}
A_{f,j}^\top
\mathcal B_a(\Lambda)^{-1}
A_{f,i}.
\]
This is the claimed derivative formula. If \(H\succeq0\), the middle summed block is
positive semidefinite, so every displayed block derivative is positive
semidefinite.
\end{proof}

\begin{proposition}[Local strict contraction near the precision fixed point]
\label{prop:app-local-precision-contraction}
There exist a monotone norm \(\|\cdot\|_\star\) on the blockwise precision
space, a radius \(\varepsilon_0>0\), and a constant \(q\in(0,1)\) such that
\[
\|\Lambda-\Lambda^\star\|_\star<\varepsilon_0
\quad\Longrightarrow\quad
\|F(\Lambda)-\Lambda^\star\|_\star
\le
q\|\Lambda-\Lambda^\star\|_\star .
\]
\end{proposition}

\begin{proof}
For a blockwise symmetric perturbation \(H=\{H_a\}_{a\in\mathcal D}\), define
\[
\|H\|_\star
:=
\max_{a\in\mathcal D}
\left\|
(\Lambda_a^\star)^{-1/2}H_a(\Lambda_a^\star)^{-1/2}
\right\|_2 .
\]
This norm is monotone on the positive semi-definite cone: if
\(0\preceq H\preceq\widetilde H\), then each normalized block is ordered in the
Loewner sense, and the spectral norm of positive semidefinite blocks is
monotone.  The definition also implies the blockwise order bound
\begin{equation}
\label{eq:app-star-order-bound-new}
-\|H\|_\star\Lambda^\star
\preceq
H
\preceq
\|H\|_\star\Lambda^\star .
\end{equation}

We show first that \(DF(\Lambda^\star)\) is strictly contractive in this norm.
Fix \(a=(f\to i)\), let \(j=j(a)\), and write
\[
S_a^\star:=\sum_{h\in N_2(j)\setminus\{f\}}\Lambda_{h\to j}^\star,
\quad
\mathcal C_a^\star:=\Lambda_j^{\rm self}+S_a^\star,
\quad
\mathcal B_a^\star:=W_f+A_{f,j}(\mathcal C_a^\star)^{-1}A_{f,j}^\top .
\]
Substituting \(H=\Lambda^\star\) into Lemma~\ref{lem:app-DF-formula} gives
\[
DF_a(\Lambda^\star)[\Lambda^\star]
=
A_{f,i}^\top(\mathcal B_a^\star)^{-1}
A_{f,j}(\mathcal C_a^\star)^{-1}
S_a^\star
(\mathcal C_a^\star)^{-1}
A_{f,j}^\top(\mathcal B_a^\star)^{-1}A_{f,i}.
\]
Because \(\Lambda_j^{\rm self}\succ0\),
\(0\preceq S_a^\star\prec\mathcal C_a^\star\), and hence
\(
(\mathcal C_a^\star)^{-1}S_a^\star(\mathcal C_a^\star)^{-1}
\prec
(\mathcal C_a^\star)^{-1}.
\)

Also \(W_f\succ0\), so
\(
A_{f,j}(\mathcal C_a^\star)^{-1}A_{f,j}^\top
\prec
W_f+A_{f,j}(\mathcal C_a^\star)^{-1}A_{f,j}^\top
=
\mathcal B_a^\star .
\)

Combining these inequalities gives
\[
DF_a(\Lambda^\star)[\Lambda^\star]
\prec
A_{f,i}^\top(\mathcal B_a^\star)^{-1}A_{f,i}
=
F_a(\Lambda^\star)
=
\Lambda_a^\star .
\]
Since there are finitely many directed messages, there exists
\(\alpha\in(0,1)\) such that
\begin{equation}
\label{eq:app-DF-star-alpha-new}
DF(\Lambda^\star)[\Lambda^\star]
\preceq
\alpha\Lambda^\star .
\end{equation}

For arbitrary \(H\), let \(t=\|H\|_\star\).  By
\eqref{eq:app-star-order-bound-new},
\(
-t\Lambda^\star\preceq H\preceq t\Lambda^\star.
\) Since \(DF(\Lambda^\star)\) is positive and linear, it is order-preserving.
Therefore
\(
DF(\Lambda^\star)[-t\Lambda^\star]
\preceq
DF(\Lambda^\star)[H]
\preceq
DF(\Lambda^\star)[t\Lambda^\star].
\)
Linearity gives
\[
-tDF(\Lambda^\star)[\Lambda^\star]
\preceq
DF(\Lambda^\star)[H]
\preceq
tDF(\Lambda^\star)[\Lambda^\star].
\]
By \eqref{eq:app-DF-star-alpha-new},
\(
DF(\Lambda^\star)[\Lambda^\star]\preceq \alpha\Lambda^\star.
\) Hence
\[
-\alpha t\Lambda^\star
\preceq
DF(\Lambda^\star)[H]
\preceq
\alpha t\Lambda^\star .
\]
Substituting \(t=\|H\|_\star\) yields
\(
-\alpha\|H\|_\star\Lambda^\star
\preceq
DF(\Lambda^\star)[H]
\preceq
\alpha\|H\|_\star\Lambda^\star .
\)
By the definition of \(\|\cdot\|_\star\), this implies
\[
\|DF(\Lambda^\star)[H]\|_\star
\le
\alpha\|H\|_\star .
\]
Thus 
$\|DF(\Lambda^\star)\|_{\star\to\star}
:=
\sup_{H\neq0}
\frac{\|DF(\Lambda^\star)[H]\|_\star}{\|H\|_\star}$$ \le\alpha<1\).

Since \(F\) is \(C^1\) near \(\Lambda^\star\), choose
\(q\in(\alpha,1)\) and then \(\varepsilon_0>0\) so that the closed
\(\|\cdot\|_\star\)-ball of radius \(\varepsilon_0\) around \(\Lambda^\star\) lies
inside \(\mathbb S_{++}\) and
\(\|DF(\Lambda)\|_{\star\to\star}\le q\) throughout that ball.  For any
\(\Lambda\) in the ball, the segment between \(\Lambda^\star\) and \(\Lambda\)
remains in the ball, and the finite-dimensional mean-value formula gives
\[
F(\Lambda)-F(\Lambda^\star)
=
\int_0^1
DF\bigl(\Lambda^\star+s(\Lambda-\Lambda^\star)\bigr)
[\Lambda-\Lambda^\star] \; ds .
\]
Taking the \(\|\cdot\|_\star\)-norm proves the claimed local contraction.
\end{proof}

\begin{theorem}[Global geometric precision convergence]
\label{cor:app-global-precision-exp}
For every admissible precision initialization, there exist constants
\(C_\Lambda<\infty\) and \(\theta_\Lambda\in(0,1)\) such that
\[
\|\Lambda^{(\ell)}-\Lambda^\star\|_\star
\le C_\Lambda\theta_\Lambda^\ell,
\qquad
\ell\ge0 .
\]
\end{theorem}

\begin{proof}
The imported exponential-entrance result \ref{thm:app-precision-imported} implies a finite entrance into the
\(\varepsilon_0\)-ball.
After that entrance time, Proposition~\ref{prop:app-local-precision-contraction} gives
a geometric tail.  Enlarging the prefactor to dominate the finite prefix yields the displayed all-time bound.
\end{proof}

\clearpage
\subsection{MAP-centered message reference}
\label{app:canonical-reference}

Let \(x_\star=A^{-1}b\).  For any precision state \(\Lambda\), define
the canonical eta reference by
\begin{equation}
\label{eq:app-canonical-eta}
\eta_{f\to i}^\star(\Lambda)
:=
\eta_i^f
-
(\Lambda_f x_{N(f)}^\star)_i
+
\Lambda_{f\to i}x_i^\star.
\end{equation}
Here \(x_{N(f)}^\star\) is the stacked vector of MAP blocks adjacent to factor \(f\), $\Lambda_f$ is the precision matrix of factor \(f\), and \((\Lambda_f x_{N(f)}^\star)_i\) denotes the block corresponding to
variable \(i\).  For unary messages, \(\eta_{u\to i}^\star=\eta_i^u\). For a binary factor \(f\) with endpoints \(i\) and \(j\),
\[
\eta_{f\to i}^\star(\Lambda)
=
\eta_i^f
-
\Lambda_{ii}^f x_i^\star
-
\Lambda_{ij}^f x_j^\star
+
\Lambda_{f\to i}x_i^\star .
\]

Let
\begin{equation}
\label{hat_lambda_i}
\widehat\Lambda_i(\Lambda):=\sum_{f\in N(i)}\Lambda_{f\to i},
\qquad
\widehat\eta_i^\star(\Lambda):=\sum_{f\in N(i)}\eta_{f\to i}^\star(\Lambda).
\end{equation}

\begin{lemma}[Canonical reference has MAP readout and propagates with precision]
\label{lem:app-canonical-reference}
For every precision state \(\Lambda\),
\[
\widehat\eta_i^\star(\Lambda)
=
\widehat\Lambda_i(\Lambda)x_i^\star,
\qquad \forall i.
\]
Moreover, if one synchronous precision sweep maps \(\Lambda\) to \(F(\Lambda)\),
then the corresponding synchronous eta Schur update maps
\(\eta^\star(\Lambda)\) to \(\eta^\star(F(\Lambda))\).  In particular,
\(\eta^\star(\Lambda^\star)\) is a fixed point of the frozen eta-message
equations at \(\Lambda^\star\), and its belief mean is \(x_\star\).
\end{lemma}

\begin{proof}
We first prove the readout identity.  
By the assembly of the global linear system,
\[
b_i=\sum_{f\in N(i)}\eta_i^f,
\qquad
(Ax_\star)_i
=
\sum_{f\in N(i)}(\Lambda_f x_{N(f)}^\star)_i,
\qquad
b_i - (A_x^{\star})_i=0.
\]
By definition~\ref{hat_lambda_i},
\[
\begin{aligned}
\widehat\eta_i^\star(\Lambda)
-
\widehat\Lambda_i(\Lambda)x_i^\star
=
\sum_{f\in N(i)}
\left[
\eta_{f\to i}^\star(\Lambda)
-
\Lambda_{f\to i}x_i^\star
\right]  
=
\sum_{f\in N(i)}
\left[
\eta_i^f
-
(\Lambda_f x_{N(f)}^\star)_i
\right].
\end{aligned}
\]

Hence
\(
\widehat\eta_i^\star(\Lambda)
-
\widehat\Lambda_i(\Lambda)x_i^\star
=
b_i-(Ax_\star)_i
=
0.
\)
Therefore
\(
\widehat\eta_i^\star(\Lambda)
=
\widehat\Lambda_i(\Lambda)x_i^\star.
\)

We now prove the propagation identity.  Consider a binary message
\(a=(f\to i)\), and let \(j\) be the other endpoint of \(f\).  Define
\[
D_a(\Lambda)
:=
\Lambda_{jj}^f
+
\sum_{h\in N(j)\setminus\{f\}}\Lambda_{h\to j}.
\]
The corresponding precision update is
\begin{equation}
\label{Fa_lambda}
F_a(\Lambda)
=
\Lambda_{ii}^f
-
\Lambda_{ij}^fD_a(\Lambda)^{-1}\Lambda_{ji}^f .
\end{equation}

The eta Schur update uses the cavity information at \(j\),
\[
r_{a,j}^\star(\Lambda)
:=
\eta_j^f+
\sum_{h\in N(j)\setminus\{f\}}\eta_{h\to j}^\star(\Lambda).
\]
We claim that
\[
r_{a,j}^\star(\Lambda)
=
D_a(\Lambda)x_j^\star+\Lambda_{ji}^f x_i^\star .
\]
Indeed, by the readout identity just proved,
\[
\sum_{h\in N(j)}\eta_{h\to j}^\star(\Lambda)
=
\left(\sum_{h\in N(j)}\Lambda_{h\to j}\right)x_j^\star .
\]
Therefore
\[
\begin{aligned}
r_{a,j}^\star(\Lambda)
&=
\eta_j^f
+
\widehat\eta_j^\star(\Lambda)
-
\eta_{f\to j}^\star(\Lambda) \\
&=
\eta_j^f
+
\widehat\Lambda_j(\Lambda)x_j^\star
-
\left[
\eta_j^f
-
\Lambda_{jj}^f x_j^\star
-
\Lambda_{ji}^f x_i^\star
+
\Lambda_{f\to j}x_j^\star
\right] \\
&=
\left(\widehat\Lambda_j(\Lambda)-\Lambda_{f\to j}+\Lambda_{jj}^f\right)x_j^\star
+
\Lambda_{ji}^f x_i^\star \\
&=
D_a(\Lambda)x_j^\star+\Lambda_{ji}^f x_i^\star .
\end{aligned}
\]

Substituting this into the eta Schur update gives
\[
\begin{aligned}
\eta_{f\to i}^{+}
&=
\eta_i^f
-
\Lambda_{ij}^fD_a(\Lambda)^{-1}r_{a,j}^\star(\Lambda) \\
&=
\eta_i^f
-
\Lambda_{ij}^fD_a(\Lambda)^{-1}
\left[
D_a(\Lambda)x_j^\star+\Lambda_{ji}^f x_i^\star
\right] \\
&=
\eta_i^f
-
\Lambda_{ij}^f x_j^\star
-
\Lambda_{ij}^fD_a(\Lambda)^{-1}\Lambda_{ji}^f x_i^\star .
\end{aligned}
\]
Using~\ref{Fa_lambda}, we obtain
\[
\begin{aligned}
\eta_{f\to i}^{+}
&=
\eta_i^f
-
\Lambda_{ij}^f x_j^\star
-
\bigl(\Lambda_{ii}^f-F_a(\Lambda)\bigr)x_i^\star \\
&=
\eta_i^f
-
\Lambda_{ii}^f x_i^\star
-
\Lambda_{ij}^f x_j^\star
+
F_a(\Lambda)x_i^\star .
\end{aligned}
\]
This is exactly the canonical message
\[
\eta_{f\to i}^\star(F(\Lambda)).
\]
Unary eta messages are unchanged, and their canonical references are also
unchanged.  Hence the full synchronous eta update maps
\[
\eta^\star(\Lambda)
\quad\text{to}\quad
\eta^\star(F(\Lambda)).
\]

Finally, setting \(\Lambda=\Lambda^\star\) and using
\(F(\Lambda^\star)=\Lambda^\star\), we get
\(
\eta^\star(\Lambda^\star)
\mapsto
\eta^\star(F(\Lambda^\star))
=
\eta^\star(\Lambda^\star).
\) 
Thus \(\eta^\star(\Lambda^\star)\) is a fixed point of the frozen eta-message equations. 

The readout identity at \(\Lambda^\star\) gives
\(
\widehat\Lambda_i(\Lambda^\star)^{-1}
\widehat\eta_i^\star(\Lambda^\star)
=
x_i^\star
\)
for every variable \(i\), so the belief mean is \(x_\star\).
\end{proof}

From now on, eta-message errors are always defined relative to the current
canonical reference:
\[
    m:=\eta-\eta^\star(\Lambda),
    \qquad
    v:=x-x_\star .
\]
At the limiting precision, this reduces to the usual frozen error
\(m=\eta-\eta^\star(\Lambda^\star)\).

\clearpage
\subsection{Frozen linear H-GBP cycle}
\label{app:linear-hgbp-clear}

Let \(P\) be the grouped prolongation matrix, and define
\begin{equation}
\label{eq:app-C-def}
C:=I-P(P^\top AP)^{-1}P^\top A .
\end{equation}
The coarse residual correction
\(
(P^\top AP)\delta z=P^\top(b-Ax),~x^+=x+P\delta z
\)
therefore maps the mean error by
\begin{equation}
\label{eq:app-galerkin-error}
    v^+=Cv,
    \qquad
    P^\top A v^+=0 .
\end{equation}
This is the Galerkin \(A\)-orthogonal projection onto
\(\operatorname{range}(P)^{\perp_A}\)~\citep{briggs2000multigrid}.

Now, let's fix the precision at \(\Lambda^\star\) -- analyze the system when all variances converge. 

Recall in the main text, we have the following matrix operators: $(J, \mathcal U)$ for repair sweep, $Q$ for \(k-1\) additional GBP sweeps,  and $\mathcal R$ for the readout.

Together with the Galerkin correction $C$ gives the limiting augmented cycle

\begin{equation}
\label{eq:app-Taug}
\begin{bmatrix}v^+\\m^+\end{bmatrix}
=
T_{{\rm aug},k}
\begin{bmatrix}v\\m\end{bmatrix},
\qquad
T_{{\rm aug},k}
=
\begin{bmatrix}
C\mathcal RQ^{k-1}J & C\mathcal RQ^{k-1}\mathcal U\\[1mm]
Q^{k-1}J & Q^{k-1}\mathcal U
\end{bmatrix}.
\end{equation}

At a post-injection cycle boundary, messages are preserved and the mean has just
been Galerkin-corrected.  Hence boundary states satisfy
\[
    v=C\mathcal Rm .
\]
Define
\[
\mathcal Y_{\rm cyc}:=\{(C\mathcal Rm,m):m\in\mathcal M\},
\qquad
\iota_{\rm cyc}(m):=(C\mathcal Rm,m),
\qquad
\pi_{\rm cyc}(v,m):=m .
\]

\begin{proposition}[Exact cycle-boundary reduction]
\label{lem:app-cycle-boundary}
\label{prop:app-exact-reduction}
The subspace \(\mathcal Y_{\rm cyc}\) is invariant under
\(T_{{\rm aug},k}\), and the reduced cycle operator is
\begin{equation}
\label{eq:app-Mk}
M_k:=\pi_{\rm cyc}T_{{\rm aug},k}\iota_{\rm cyc}
=Q^{k-1}(JC\mathcal R+\mathcal U).
\end{equation}
Moreover,
\[
T_{{\rm aug},k}\iota_{\rm cyc}=\iota_{\rm cyc}M_k .
\]
\end{proposition}

\begin{proof}
Substituting \(v=C\mathcal Rm\) into \eqref{eq:app-Taug} gives
\[
T_{{\rm aug},k}\iota_{\rm cyc}(m)
=
\left(
C\mathcal RQ^{k-1}(JC\mathcal R+\mathcal U)m,
Q^{k-1}(JC\mathcal R+\mathcal U)m
\right)
=
\iota_{\rm cyc}(M_km).
\]
\end{proof}

\begin{lemma}[Spectral equivalence]
\label{lem:app-rho-equality}
\[
    \rho(M_k)=\rho(T_{{\rm aug},k}).
\]
\end{lemma}

\begin{proof}
The map \(\iota_{\rm cyc}\) is an isomorphism from \(\mathcal M\) to
\(\mathcal Y_{\rm cyc}\), so \(M_k\) is similar to
\(T_{{\rm aug},k}|_{\mathcal Y_{\rm cyc}}\).  Also,
\(T_{{\rm aug},k}\) maps the whole augmented space into
\(\mathcal Y_{\rm cyc}\); hence every nonzero eigenvector of
\(T_{{\rm aug},k}\) lies in \(\mathcal Y_{\rm cyc}\).  The two operators therefore have the same nonzero eigenvalues.
\end{proof}

\clearpage
\subsection{From the initial precision to the limiting cycle}
\label{app:limiting-phase}

For cycle \(t\), let \(\Lambda_t^{[r]}:=\Lambda^{(tk+r)}\),
\(r=0,\ldots,k\).  The operators
\(J_t,\mathcal U_t,\mathcal R_t\) are evaluated at
\(\Lambda_t^{[0]}\), while
\(Q_{t,r}\) is evaluated at \(\Lambda_t^{[r]}\).  Because the canonical reference
\(\eta^\star(\Lambda)\) propagates with the precision update
(Lemma~\ref{lem:app-canonical-reference}), the eta-error dynamics are exactly
homogeneous at every precision state.

\begin{definition}[Exact cocycle and limiting operator]
\label{def:app-exact-cocycle}
The exact cycle-boundary cocycle is
\begin{equation}
\label{eq:app-exact-cocycle}
\widetilde M_{k,t}
:=
Q_{t,k-1}\cdots Q_{t,1}
\bigl(J_tC\mathcal R_t+\mathcal U_t\bigr).
\end{equation}
The limiting operator is
\[
M_k=Q^{k-1}(JC\mathcal R+\mathcal U),
\]
where \(Q,J,\mathcal U,\mathcal R\) are evaluated at \(\Lambda^\star\).
\end{definition}

\begin{proposition}[Exact cycle-boundary dynamics]
\label{prop:app-exact-cycle}
Let \(m_t\) be the directed eta-message error at the post-injection boundary of
cycle \(t\), defined relative to \(\eta^\star(\Lambda_t^{[0]})\).  Then
\[
    v_t=C\mathcal R_t m_t,
    \qquad
    m_{t+1}=\widetilde M_{k,t}m_t .
\]
\end{proposition}

\begin{proof}
Before injection, the readout error is \(\mathcal R_t m_t\); the Galerkin step
maps it to \(C\mathcal R_t m_t\), proving the boundary relation.  The first
sweep gives \((J_tC\mathcal R_t+\mathcal U_t)m_t\), and the remaining
\(k-1\) message-consistent sweeps multiply successively by
\(Q_{t,1},\ldots,Q_{t,k-1}\).  This is exactly
\eqref{eq:app-exact-cocycle}.
\end{proof}

\begin{lemma}[Local Lipschitz dependence of frozen operators]
\label{lem:app-local-lipschitz-operators}
There is a compact neighborhood \(\mathcal K_0\subset\mathbb S_{++}\) of
\(\Lambda^\star\) and a constant \(L_{\rm op}<\infty\) such that, for all
\(\Lambda,\widetilde\Lambda\in\mathcal K_0\),
\[
\|J(\Lambda)-J(\widetilde\Lambda)\|
+
\|\mathcal U(\Lambda)-\mathcal U(\widetilde\Lambda)\|
+
\|\mathcal R(\Lambda)-\mathcal R(\widetilde\Lambda)\|
+
\|Q(\Lambda)-Q(\widetilde\Lambda)\|
\le
L_{\rm op}\|\Lambda-\widetilde\Lambda\|_\star .
\]
The same operator families are uniformly bounded on \(\mathcal K_0\).
\end{lemma}

\begin{proof}
Every block of these operators is a rational expression in the precision blocks,
formed from sums, products, Schur denominators, and inverses.  On a compact
positive-definite neighborhood of \(\Lambda^\star\), all denominators are
uniformly nonsingular.  

The maps are therefore \(C^1\), Lipschitz, and uniformly
bounded on \(\mathcal K_0\).
\end{proof}

\begin{lemma}[Exponential convergence of the exact cocycle]
\label{lem:app-exact-cocycle-exp}
For every admissible precision initialization and fixed \(k\), there exist
constants \(C_{M,k}<\infty\) and \(\theta_{M,k}\in(0,1)\) such that
\[
\|\widetilde M_{k,t}-M_k\|
\le C_{M,k}\theta_{M,k}^t,
\qquad t\ge0 .
\]
Consequently, with \(E_t:=\widetilde M_{k,t}-M_k\),
\[
    \sum_{t=0}^\infty \|E_t\|<\infty .
\]
\end{lemma}

\begin{proof}
By Corollary~\ref{cor:app-global-precision-exp},
\(\|\Lambda^{(\ell)}-\Lambda^\star\|_\star\le
C_\Lambda\theta_\Lambda^\ell\).  After a finite prefix, all
within-cycle precision states lie in \(\mathcal K_0\).  Lemma~\ref{lem:app-local-lipschitz-operators}
then gives
\[
\|J_t-J\|+
\|\mathcal U_t-\mathcal U\|+
\|\mathcal R_t-\mathcal R\|+
\sum_{r=1}^{k-1}\|Q_{t,r}-Q\|
\le
C_k(\theta_\Lambda^k)^t .
\]
A telescoping expansion of the product
\(Q_{t,k-1}\cdots Q_{t,1}\) against \(Q^{k-1}\), together with uniform
boundedness of all factors, yields the displayed bound.  The finite prefix is
absorbed into \(C_{M,k}\), and the geometric bound is summable.
\end{proof}

\begin{proposition}[Stability under a vanishing perturbation]
\label{prop:app-summable-perturbation}
Let \(z_{t+1}=(A+E_t)z_t\) in finite dimension and \(E_t\to0\).  If
\(\rho(A)<1\), then \(z_t\to0\) for every \(z_0\).  Moreover, for every
\(q\in(\rho(A),1)\), there is \(C_q<\infty\) such that
\[
\|z_t\|\le C_q q^t\|z_0\| .
\]
In particular, the conclusion holds if \(\sum_t\|E_t\|<\infty\).
\end{proposition}

\begin{proof}
Choose an induced norm with \(\|A\|<q\), possible by the standard adapted-norm
construction for finite-dimensional matrices.  Since \(E_t\to0\), the tail
satisfies \(\|A+E_t\|\le q\).  The finite prefix is absorbed into the constant.
\end{proof}

\clearpage
\subsection{Global convergence criterion and MAP target}
\label{app:necessary-sufficient}

\begin{definition}[Admissible initialization]
\label{def:app-admissible-init}
An admissible initialization is a post-injection cycle-boundary state with
\(\Lambda^{(0)}\succeq0\) and finite \(x^{(0)}\), \(\eta^{(0)}\), for which all
GBP Schur complements and H-GBP updates are well defined.
\end{definition}

\begin{theorem}[Global necessary and sufficient criterion for H-GBP]
\label{thm:app-global-criterion}
\label{thm:app-Mk-iff}
Fix the cycle length \(k\).  The following statements are equivalent:
\begin{enumerate}[label=(\alph*)]
\item For every admissible initialization, the full H-GBP trajectory converges to
      the MAP-centered fixed point.
\item For every admissible precision initialization, the exact cocycle
      \(m_{t+1}=\widetilde M_{k,t}m_t\) satisfies \(m_t\to0\) for every
      \(m_0\in\mathcal M\).
\item \(\rho(M_k)<1\).
\item \(\rho(T_{{\rm aug},k})<1\).
\end{enumerate}
When these conditions hold, for every \(q\in(\rho(M_k),1)\) there exists
\(C_q<\infty\) such that
\[
\|m_t\|\le C_q q^t\|m_0\|,
\qquad
\|v_t\|\le C_q q^t\|m_0\|,
\qquad t\ge0,
\]
after enlarging \(C_q\) to absorb the finite prefix.
\end{theorem}

\begin{proof}
By Proposition~\ref{prop:app-exact-cycle}, the cycle-boundary dynamics are
\[
    m_{t+1}=\widetilde M_{k,t}m_t,
    \qquad
    v_t=C\mathcal R_t m_t .
\]
The readout operators are uniformly bounded along the compact precision
trajectory, so \(m_t\to0\) is equivalent to \((v_t,m_t)\to(0,0)\) at cycle
boundaries.  Within-cycle states are obtained by finitely many uniformly bounded
maps, so cycle-boundary convergence is equivalent to full-trajectory convergence.

If \(\rho(M_k)<1\), Lemma~\ref{lem:app-exact-cocycle-exp} gives
\(\widetilde M_{k,t}=M_k+E_t\) with \(E_t\to0\), in fact summably.  Applying
Proposition~\ref{prop:app-summable-perturbation} gives geometric convergence of
\(m_t\), and uniform boundedness of \(C\mathcal R_t\) gives the same rate for
\(v_t\).

Conversely, suppose the full hierarchy converges for every admissible
initialization.  Choose \(\Lambda^{(0)}=\Lambda^\star\).  Then the precision is
frozen, \(\widetilde M_{k,t}=M_k\), and the message recursion is
\(m_{t+1}=M_km_t\).  Convergence for every \(m_0\) in a finite-dimensional
linear system is equivalent to \(\rho(M_k)<1\).  Finally,
Lemma~\ref{lem:app-rho-equality} gives
\(\rho(M_k)=\rho(T_{{\rm aug},k})\).
\end{proof}

\begin{theorem}[Exact MAP target]
\label{thm:app-exact-map-target}
If the equivalent conditions of Theorem~\ref{thm:app-global-criterion} hold,
then for every admissible initialization,
\[
    x_t\to x_\star=A^{-1}b,
    \qquad
    \eta_t\to\eta^\star(\Lambda^\star),
    \qquad
    \Lambda_t\to\Lambda^\star .
\]
Thus the fine mean produced by H-GBP converges to the exact MAP solution of the
base linear Gaussian factor graph.
\end{theorem}

\begin{proof}
Precision convergence follows from Theorem~\ref{thm:app-precision-imported} and
Proposition~\ref{prop:app-hierarchy-precision}.  Theorem~\ref{thm:app-global-criterion}
gives \(m_t\to0\) and \(v_t\to0\).  Since
\(v_t=x_t-x_\star\), this proves \(x_t\to x_\star\).  

Since the canonical
reference \(\eta^\star(\Lambda)\) is continuous in \(\Lambda\) and
\(\Lambda_t\to\Lambda^\star\), we also have
\(\eta_t=\eta^\star(\Lambda_t)+m_t\to\eta^\star(\Lambda^\star)\).
\end{proof}

\clearpage
\subsection{H-GBP on manifolds}
\label{app:hgbp-manifolds}

Using an iterative inner solver for each linearized least-squares problem is
the truncated/inexact Gauss--Newton construction
\citep{gratton2007approximate,dembo1982inexact}.
H-GBP fits this construction: at outer iteration \(t\), it approximately solves
\begin{equation}
\label{eq:app-manifold-normal}
A_t\xi=b_t,\qquad
A_t=J_t^\top J_t,\qquad
b_t=-J_t^\top r_t,
\end{equation}
and retracts \(X_{t+1}=X_t\oplus\widehat\xi_t\).
Here \(F(X)=\tfrac12\lVert r(X)\rVert_2^2\), and derivatives and tangent norms
are expressed in gauge-fixed local coordinates.
We now give a direct H-GBP result for local nonlinear convergence and its
outer rate. Because H-GBP approaches the tangent solution iteratively, the
only added algorithmic condition is the standard relative-residual stopping
rule
\begin{equation}
\label{eq:app-manifold-stopping}
\eta_t:=
\frac{\lVert A_t\widehat\xi_t-b_t\rVert_{A_t^{-1}}}
     {\lVert b_t\rVert_{A_t^{-1}}}
\le\bar\eta,
\qquad
0<\bar\eta<\frac{1-\gamma_\star}{1+\gamma_\star}.
\end{equation}
If \(b_t=0\), the algorithm terminates.
No additional structural assumption is introduced: the outer conditions are
the standard local Gauss--Newton conditions, and the inner conditions are
exactly those of our linear H-GBP theorem.

\begin{lemma}[Standard inexact Gauss--Newton]
\label{lem:app-inexact-gn}
Let \(X_\star\) be a stationary point. Assume that the whitened residual
\(r\) is twice continuously differentiable, the gauge-fixed matrix
\(A_\star=J_\star^\top J_\star\) is positive definite, and the retraction
is smooth. Define
\[
S_\star:=\nabla^2F(X_\star)-A_\star,
\qquad
\gamma_\star:=
\lVert A_\star^{-1/2}S_\star A_\star^{-1/2}\rVert_2<1.
\]
Let \(\xi_t^\star=A_t^{-1}b_t\) be the exact Gauss--Newton step. If
\[
\frac{\lVert\widehat\xi_t-\xi_t^\star\rVert_{A_t}}
     {\lVert\xi_t^\star\rVert_{A_t}}
\le\eta_t\le\bar\eta,
\]
with \(\bar\eta\) as in~\eqref{eq:app-manifold-stopping}, then, for
\(X_0\) sufficiently close to \(X_\star\),
\begin{equation}
\label{eq:app-manifold-limsup-rate}
\limsup_{t\to\infty}
\frac{\lVert X_{t+1}\ominus X_\star\rVert_{A_\star}}
     {\lVert X_t\ominus X_\star\rVert_{A_\star}}
\le
\gamma_\star+(1+\gamma_\star)\limsup_{t\to\infty}\eta_t.
\end{equation}
Hence \(\eta_t\le\bar\eta<(1-\gamma_\star)/(1+\gamma_\star)\) gives
local \(Q\)-linear convergence.
\end{lemma}

\begin{proof}
This follows from the standard perturbation theory of locally contractive
fixed-point iterations \citep[Chapter~10]{ortega1970iterative}:
the exact Gauss--Newton map has local asymptotic contraction factor
\(\gamma_\star\), while a relative step perturbation \(\eta_t\)
contributes at most \((1+\gamma_\star)\eta_t\) asymptotically.
\end{proof}

\begin{lemma}[Residual--error identity]
\label{lem:app-manifold-residual-error}
For \(A_t\succ0\) and \(b_t\ne0\),
\begin{equation}
\label{eq:app-manifold-residual-error}
\frac{\lVert A_t\widehat\xi_t-b_t\rVert_{A_t^{-1}}}
     {\lVert b_t\rVert_{A_t^{-1}}}
=
\frac{\lVert\widehat\xi_t-\xi_t^\star\rVert_{A_t}}
     {\lVert\xi_t^\star\rVert_{A_t}}.
\end{equation}
\end{lemma}

\begin{proof}
Since \(A_t\xi_t^\star=b_t\), we have
\(A_t\widehat\xi_t-b_t=A_t(\widehat\xi_t-\xi_t^\star)\).
The identity follows directly from the definitions of the
\(A_t\)- and \(A_t^{-1}\)-norms.
\end{proof}

\begin{theorem}[Local convergence of nonlinear H-GBP]
\label{thm:app-nonlinear-hgbp}
Assume the conditions of Lemma~\ref{lem:app-inexact-gn} hold.
At every outer iteration before termination, suppose the tangent graph
satisfies our existing linear H-GBP assumptions and
\[
\rho(M_{k,t})<1,
\]
where \(M_{k,t}\) is the H-GBP cycle operator for the current linearization.
Run H-GBP until the stopping rule~\eqref{eq:app-manifold-stopping} holds.
Then \(X_{t+1}=X_t\oplus\widehat\xi_t\) converges locally to \(X_\star\).
\end{theorem}

\begin{proof}
Theorems~\ref{thm:app-global-criterion} and~\ref{thm:app-exact-map-target}
imply that \(\rho(M_{k,t})<1\) makes H-GBP converge geometrically to
\(\xi_t^\star=A_t^{-1}b_t\), so the stopping tolerance is reached in
finitely many cycles. Lemma~\ref{lem:app-manifold-residual-error} converts
that tolerance exactly into the forcing condition of
Lemma~\ref{lem:app-inexact-gn}, which proves the claim.
\end{proof}

\begin{corollary}[H-GBP-controlled outer rate]
\label{cor:app-hgbp-outer-rate}
Under the assumptions of Theorem~\ref{thm:app-nonlinear-hgbp}, for any
\(q_t\in(\rho(M_{k,t}),1)\), our linear theorem gives a finite \(K_t\)
such that after \(T\) H-GBP cycles,
\[
\eta_t^{(T)}\le K_tq_t^T.
\]
Hence the stopping rule is reached after finitely many cycles.
For example, it is sufficient to take
\[
T_t\ge
\left\lceil
\frac{\log(K_t/\bar\eta)}{-\log q_t}
\right\rceil
\qquad\text{whenever }K_t>\bar\eta.
\]
The resulting nonlinear outer rate satisfies
\begin{equation}
\label{eq:app-hgbp-outer-rate}
q_{\mathrm{outer}}
\le\gamma_\star+(1+\gamma_\star)\bar\eta<1.
\end{equation}
Thus H-GBP's geometric inner convergence guarantees finite attainment of the
prescribed forcing tolerance, while that tolerance explicitly controls the
nonlinear outer convergence rate.
\end{corollary}

\begin{proof}
The finite-cycle bound follows from the linear H-GBP theorem.
Substituting \(\eta_t\le\bar\eta\) into Lemma~\ref{lem:app-inexact-gn}
gives the outer-rate bound.
\end{proof}

\clearpage
\section{Algorithm}
\label{app:algdetail}

\begin{algorithm}[h]
\caption{Hierarchical Gaussian BP}
\label{alg:fresh-hgbp}
\begin{algorithmic}[1]
\Require Fixed linear Gaussian factor graph $\mathcal G$ with factor parameters $\{(\eta_f,\Lambda_f)\}_{f\in\mathcal F}$; groups $\{G\}$; ranks $\{r_G\}$; sweeps $k$; cycles $T$.
\State Assemble $A$ and $b$; initialize all factor-to-variable messages to zero.
\For{each group $G$}
    \State Build $H_G$ from factors satisfying $N(f)\subseteq G$ using~\eqref{eq:fresh-local-operator}.
    \State Store the $r_G$ lowest eigenvectors of $H_G$ as the local basis $P_G$.
\EndFor
\State Assemble $P=\operatorname{blkassemble}\{P_G\}_G$ and factorize $A_c=P^\top AP$.
\For{$t=0,\ldots,T-1$}
    \For{$s=1,\ldots,k$}
        \State Compute variable-to-factor messages by~\eqref{eq:variable-to-factor}.
        \State Compute factor-to-variable messages $\Lambda_{f\to i}^{+}$,
$\eta_{f\to i}^{+}$ by the Schur update~\eqref{eq:schur-message}.
        \State Update all variable beliefs by~\eqref{eq:variable-update} (and means $x_i$ optionally).
    \EndFor
    \State Stack $x=(x_1^{\top},\ldots,x_N^{\top})^{\top}$ and compute $r_c=P^\top(b-Ax)$.
    \State Solve $A_c\delta z=r_c$ and prolong $\delta x=P\delta z$.
    \State Inject $x_i\leftarrow x_i+\delta x_i$ and $\widehat\eta_i\leftarrow\widehat\Lambda_i x_i$ for all variables. Keep all $(\eta_{f\to i},\Lambda_{f\to i})$ unchanged.
\EndFor
\State \Return $x=(x_1^{\top},\ldots,x_N^{\top})^{\top}$.
\end{algorithmic}
\end{algorithm}

\clearpage
\section{Additional Details}
\label{app:additional}

\subsection{Implementation Details}
\label{app:impl}

All timings are CPU-only, on an Intel Core i9-13900HX system (24 physical cores, 32 logical processors, 64\,GB RAM). A 16-thread run specifies a worker budget, not 16 physical cores. The released configurations specify one policy for SE(2), one for SE(3), and one for BA; none selects parameters by dataset name.

\paragraph{Precision reuse and packed messages.}
Within a fixed Gaussian model, precision updates are independent of the canonical vectors. Once the precision stopping rule is met, we cache the Schur gain
\[
 B_{f\to i}=\Lambda_{ij}^{f}
 (\Lambda_{jj}^{f}+\Lambda_{j\setminus f})^{-1},
 \qquad
 \eta_{f\to i}^{+}=\eta_i^f-B_{f\to i}(\eta_j^f+\eta_{j\setminus f}).
\]
Subsequent eta-only sweeps reuse both this gain and the belief precisions, replacing repeated small factorizations by matrix--vector products. This changes the per-message arithmetic from \(O(d^3)\) to \(O(d^2)\) for block dimension \(d\). Caches are refreshed when the linearized model changes. Numerical freezing uses the task-specific tolerances below, rather than assuming an exact fixed point after a prescribed number of sweeps.

\paragraph{Local basis construction.}
For PGO, a group \(G\) assembles \(H_G\in\mathbb{R}^{d|G|\times d|G|}\) from its internal factors and priors, and retains the \(r_G\) lowest eigenmodes as \(P_G\). Here \(d=3\) for SE(2) and \(d=6\) for SE(3), with ranks \(4\) and \(12\), respectively (capped by the group dimension). SE(2) uses block inverse iteration with Rayleigh--Ritz extraction and an eigen-residual acceptance test, warm-started from the preceding basis when available. SE(3) first exploits exact band structure with a selected-eigenpair solver when the local block is sufficiently narrow-banded, and otherwise attempts the partial solve. Unaccepted partial solves fall back to a full symmetric eigendecomposition. These backends target the same low-mode space; local bases are rebuilt at each relinearization.

\paragraph{Parallel and distributed structure.}
Fine-level factor and belief passes, as well as per-group basis construction, admit parallel execution with local data access. Contiguous message slots, fixed-size block kernels, preallocated workspaces, and persistent worker regions reduce allocation and synchronization overhead; restriction accumulates group contributions and prolongation applies local basis blocks. The one-thread PGO path avoids launching parallel worker teams. Our implementation is shared-memory: reductions and the reduced solve still provide global coupling. In preliminary experiments, GBP on the induced second-level factor graph was slower than a direct solve. A likely reason is the much smaller coarse-system dimension, which makes a factorization reused across corrections cheaper than repeated coarse message sweeps. We therefore use sparse Cholesky for the PGO coarse system and dense \(LDL^\top\) for the small BA coarse system, while retaining loopy GBP on the fine message graph.

\subsection{Linear Experiment Details}
\label{app:exp-details}

The convergence studies use fixed Gaussian systems, with a direct solution as the MAP reference. Figure~\ref{fig:linear_hgbp_vs_gbp} varies the smoothing budget on a synthetic translation graph; Fig.~\ref{fig:all}(a,b) examines coarse spaces and spectral contraction. The companion fill-in timing study in Fig.~\ref{fig:all}(c) is described separately below.

\subsubsection{Two-lap translation graph}
\label{sec:supp_two_lap}

The graph has \(200\) two-dimensional position variables, arranged on two circular laps of radius \(10\), with \(100\) positions per lap and random seed \(11\). Odometry initialization integrates consecutive relative displacements with Gaussian noise \(\mathcal N(0,0.10^2 I_2)\). Within-group chain factors use those measurements; connectors, skips, and cross-lap closures use ground-truth displacements perturbed by \(\mathcal N(0,0.01^2 I_2)\). Twenty groups contain ten consecutive variables each. A soft prior of weight \(10^{-6}\) anchors the first position. Each group retains four modes (two scalar modes in each coordinate), reducing dimension \(400\) to \(80\). Table~\ref{tab:supp_two_lap_edges} specifies the edge counts and information weights; the displayed realization selects \(135\) within-group skips from \(160\) eligible candidates.

\begin{table}[t]
\centering
\caption{Edge construction for Fig.~\ref{fig:linear_hgbp_vs_gbp}.}
\label{tab:supp_two_lap_edges}
\resizebox{\linewidth}{!}{%
\begin{tabular}{lccc}
\toprule
Edge type & Construction & Count & Weight \\
\midrule
Local chain & adjacent poses within each group & 180 & 1.0 \\
Inter-group connector & last pose of a group to first pose of next group & 20 & 1.0 \\
Short-range skip & selected within-group skips, distance \(2\)--\(6\) & 135 & 0.8 \\
Long-range loop closure & cross-lap edges with phase jitter \(\delta\in[-3,3]\) nodes & 16 & 0.2 \\
\bottomrule
\end{tabular}%
}

\end{table}

We vary \(k\in\{20,30,\ldots,100\}\) and \(c\in\{0,1,\ldots,20\}\). One GBP cycle contains \(k\) synchronous sweeps; H-GBP adds a coarse correction to the same fine-sweep budget. Both start at the same odometry estimate. The surface plots
\[
 \log_{10}\!\left(\frac{\|e_c-e_\star\|_A}{\|e_0-e_\star\|_A}\right).
\]
Thus the comparison isolates the benefit of abstraction and recovery at equal numbers of fine sweeps, not equal wall time. The energy norm measures the quadratic-system error; because the anchor is weak, it should not be interpreted as an unweighted trajectory MSE.

\subsubsection{Coarse-space ablation: Fig.~\ref{fig:all}(a)}
\label{sec:supp_abstraction}

This panel freezes the first SE(3) linearization of Sphere, \(He=\eta\), and uses \(e_{\rm MAP}=H^{-1}\eta\) as the reference. Ordered groups contain \(20\) poses. Low-rank abstraction uses the lowest modes of the block assembled from priors and factors fully contained in a group. Message-\(\lambda\) abstraction additionally includes incoming boundary-message precision blocks. The algebraic multigrid comparator (labelled AMD in the figure) instead uses block strength of connection to select coarse pose nodes and interpolate neighboring fine nodes~\citep{xu2017algebraic}.

Starting from \(e=0\), all three methods apply a coarse correction, \(50\) synchronous GBP sweeps, and a second coarse correction. Each correction solves \((P^\top HP)\delta z=P^\top(\eta-He)\) and injects \(P\delta z\). Table~\ref{tab:supp_abstraction_error} reproduces the final relative Euclidean errors shown in the panel. The low-rank space yields a substantially smaller remaining error on this system.

\begin{table}[t]
\centering
\caption{Sphere final relative error \(\|e-e_{\rm MAP}\|_2/\|e_{\rm MAP}\|_2\), as displayed in Fig.~\ref{fig:all}(a).}
\label{tab:supp_abstraction_error}
\begin{tabular}{lc}
\toprule
Coarse abstraction & Relative error \\
\midrule
AMD & 0.928 \\
Message-\(\lambda\) & 0.899 \\
Low-rank & 0.335 \\
\bottomrule
\end{tabular}
\end{table}

\subsubsection{Spectral contraction: Fig.~\ref{fig:all}(b)}
\label{sec:supp_spectral_radius}

The spectral experiment again uses \(200\) translation variables, group size \(10\), rank \(4\), and \(k=50\). Only the short- and long-range connectivity varies (Table~\ref{tab:supp_spectral_radius_variants}). We plot successive-cycle energy-error ratios \(\|e_c-e_\star\|_A/\|e_{c-1}-e_\star\|_A\), for \(c=1,\ldots,20\), against the estimated fixed-precision spectral radius \(\rho(T_{{\rm aug},50})=\rho(M_{50})\). The ratios approach the predicted asymptotic rates as precision messages stabilize. Faster initial contraction is compatible with the theory: a spectral radius describes the asymptotic rate, not a universal one-step norm bound.

\begin{table}[t]
\centering
\caption{Graph variants used in the spectral-radius contraction test.}
\label{tab:supp_spectral_radius_variants}
\resizebox{\linewidth}{!}{%
\begin{tabular}{lccccccc}
\toprule
Case 
& \(p_{\rm skip}\) 
& \(n_{\rm loop}\) 
& \(w_{\rm skip}\) 
& \(w_{\rm loop}\) 
& Short-range 
& Long-loop 
& \(\rho(T_{{\rm aug},50})\) \\
\midrule
Dense local / few loops 
& 1.0 & 8  & 1.2 & 0.05 & 160 & 8  & 0.393 \\
Moderate local / moderate loops 
& 0.8 & 16 & 0.8 & 0.20 & 135 & 16 & 0.428 \\
Sparse local / more loops 
& 0.3 & 24 & 0.4 & 1.00 & 44  & 24 & 0.657 \\
\bottomrule
\end{tabular}%
}

\end{table}

\subsubsection{Fill-in sensitivity: Fig.~\ref{fig:all}(c)}
\label{sec:supp_fill_in}

The fill-in diagnostic concerns the linear system \(A_0\delta=b_0\) obtained by the first linearization at the initial odometry estimate. The Globe1k variants all have matrix dimension \(6144\) and \(98{,}544\) lower-triangular nonzeros, but different elimination fill-in. The accompanying recorded runtimes cover ten nonlinear outer iterations, rather than only this first linear solve. H-GBP uses five cycles, fifty pre-sweeps, group size \(20\), rank \(12\), and one thread throughout. Each method's runtime is normalized by its own smallest-fill case (Table~\ref{tab:supp_fill_in}). Across these variants, Cholesky time grows by \(1.883\times\), versus \(1.075\times\) for H-GBP. This is an empirical sensitivity comparison at fixed problem size and sparsity, not an asymptotic complexity claim.

\begin{table}[t]
\centering
\caption{Runtime scaling under increasing sparse-matrix fill-in.}
\label{tab:supp_fill_in}
\resizebox{\linewidth}{!}{%
\begin{tabular}{lccccc}
\toprule
Mode 
& Fill-in ratio 
& H-GBP time (s) 
& Cholesky time (s) 
& H-GBP growth 
& Cholesky growth \\
\midrule
local\_diag          & 5.066 & 1.459 & 0.475 & 1.000 & 1.000 \\
medium\_chords      & 6.141 & 1.504 & 0.603 & 1.031 & 1.269 \\
separator\_weave    & 6.781 & 1.500 & 0.682 & 1.028 & 1.436 \\
mixed\_chords       & 6.864 & 1.549 & 0.720 & 1.061 & 1.516 \\
hemisphere\_bridges & 8.102 & 1.568 & 0.894 & 1.075 & 1.883 \\
\bottomrule
\end{tabular}%
}

\end{table}

\subsection{SE2/SE3 Pose-Graph Experiment Details}
\label{app:exp-details2}

\paragraph{Baselines and protocol.}
The Cholesky reference uses the official g2o \texttt{LinearSolverCholmod} kernel~\citep{kuemmerle2011g2o} inside our matched analytic Lie-log front end, with block ordering and symbolic reuse. The archived CHOLMOD build is simplicial: its dominant numerical factorization is serial, so additional front-end threads do not accelerate that stage. g2o-PCG uses block-preconditioned conjugate gradients (relative tolerance \(10^{-6}\), iteration cap equal to the scalar system dimension). Both use twenty Gauss--Newton outer iterations.

AMM-PGO~\citep{fan2024mm_pgo} combines distributed majorization--minimization with Nesterov acceleration and adaptive restart; we use the official implementation with a parallel-node adapter and sixteen partitions. Cycle-PGO~\citep{bai2021cycle} optimizes relative poses using a minimum cycle basis and cycle-consistency constraints. Schwarz-PGO uses sixteen subdomains, one-hop overlap, a GDSW coarse space, and PCG tolerance \(10^{-8}\)~\citet{tee2026schwarz}. Cycle-PGO and Schwarz-PGO allow up to twenty nonlinear iterations. AMM-PGO retains its recorded native budgets: \(100\) iterations for FR079 and Globe100k, \(20\) for Grid, and \(300\) otherwise.

All methods receive the same released graph files, including the offline PSD-conditioned Cubicle graph and the generated Globe graphs. H-GBP, Cholesky, PCG, and Schwarz use Huber threshold \(5\). AMM-PGO retains its native scalar-weighted chordal robust objective and initialization; Cycle-PGO retains its native non-robust objective. For comparison, final poses are rescored using the same full-information, half-quadratic Lie-log cost, so displayed costs are not the Huber objectives used internally. Cycle estimates are reconstructed into vertex poses before scoring. Timings are measured wall times with native boundaries retained; in particular, Schwarz totals include parsing. Sixteen threads denotes the available worker budget, not guaranteed parallelism in every stage.

\paragraph{Unified H-GBP settings.}
SE(2) uses \(20\) outer iterations, \(c=3\), \(k=100\), group size \(20\), and \(r=4\); SE(3) uses \(20\), \(c=5\), \(k=50\), group size \(20\), and \(r=12\). Analytical Jacobians produce each tangent system. Loopy GBP pre-smoothing and a reduced sparse direct solve alternate within every outer iteration, followed by retraction. Precision freezing requires three consecutive checks below relative tolerance \(10^{-8}\). Checks initially occur every five sweeps (the SE(2) interval may back off to twenty-five); frozen precisions are rechecked every twenty-five sweeps and reactivated if necessary. These rules are shared across datasets within each space.

\paragraph{Objective trajectories.}
Figure~\ref{fig:nine} retains the archived objective-versus-outer comparison with the direct Gauss--Newton reference as a qualitative convergence illustration, separate from the current timing runs. H-GBP consistently reduces the objective across the nine graphs and avoids the pronounced early spikes seen in the direct reference on Cubicle and Grid.

\begin{figure}[t]
    \centering
    \includegraphics[width=\linewidth]{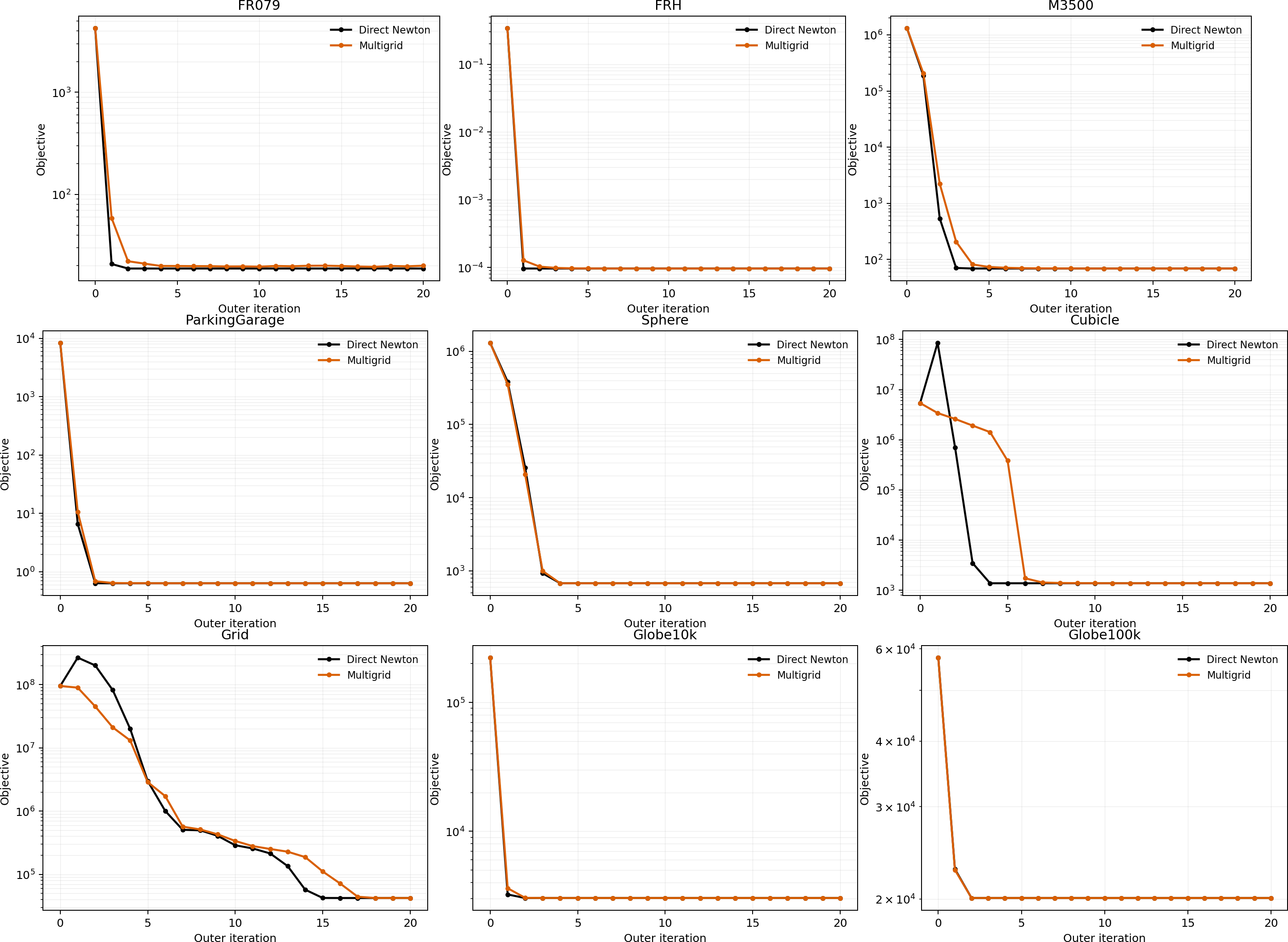}
    \caption{Archived objective trajectories of H-GBP and the direct Gauss--Newton reference over nonlinear outer iterations.}
    \label{fig:nine}
\end{figure}

\subsection{Bundle Adjustment Experiment Details}
\label{app:exp-details-ba}

\paragraph{Baselines and common objective.}
Ceres Sparse uses sparse Cholesky on the reduced camera system (\texttt{SPARSE\_SCHUR}, SuiteSparse/CHOLMOD); Ceres Iterative uses implicit Schur PCG with a Schur--Jacobi preconditioner~\citep{ceres}. RootBA-QR eliminates landmarks by Householder QR in square-root form and iteratively solves the reduced camera problem~\citep{demmel2021rootba}. PowerBA approximates the inverse Schur complement by a truncated power series, here of order ten~\citep{weber2023poba}. We use their released implementations through the common BAL/RootBA adapters, with double precision, identical input files, camera model, and normalization. All receive a budget of twenty nonlinear iterations and retain native stopping. Cost is the sum of squared two-dimensional reprojection residuals. Table timings exclude loading and common input normalization; Ceres reports \texttt{Solve} time, excluding wrapper problem construction, whereas H-GBP, RootBA, and PowerBA include their internal setup.

\begin{figure}[t]
    \centering
    \includegraphics[width=\linewidth,height=0.83\textheight,keepaspectratio]{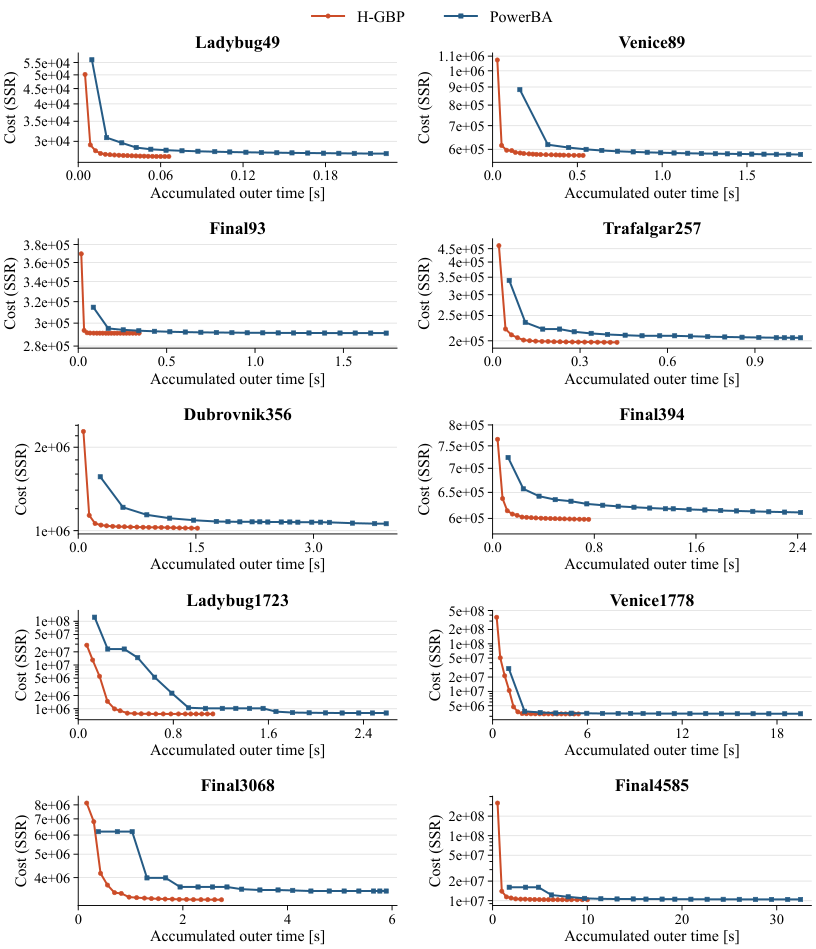}
    \caption{H-GBP versus PowerBA on BA, both with sixteen threads. Markers show outer attempts \(1\)--\(20\); vertical axes show reprojection cost (SSR) on logarithmic scales. Time is accumulated inside outer iterations, excluding one-time setup. Curves use recorded durations without rescaling.}
    \label{fig:ba_cost_time}
\end{figure}

\paragraph{A single policy across scenes.}
Cameras are aggregated into connected groups of target size \(\max(2,\lceil n_c/24\rceil)\), and each group retains nine modes when abstracted. The damping coefficient for the fine graph is one. Precision messages are computed for each camera system until relative change is below \(10^{-6}\). We then reuse the cached eta maps across H-GBP applications. Grouping and sparsity follow the same graph rules for every scene.

\paragraph{Cost versus iteration time.}
Figure~\ref{fig:ba_cost_time} compares the actual per-outer histories of H-GBP and PowerBA on all ten BA benchmarks. H-GBP reaches lower final cost on these runs while spending less time in the nonlinear iterations. To avoid mixing timing boundaries, the horizontal axis sums the recorded outer durations for both methods, excluding one-time setup, loading, and normalization; the main table separately reports solver totals. Each point is one outer attempt, and rejected PowerBA trials retain the previous accepted cost while their elapsed time is still counted. Each H-GBP curve is a complete run with median total runtime among three repeats, not a curve assembled from per-iteration minima.

\vspace{5em}
\section{Discussion}
\label{app:discussion}
\paragraph{Toward multiple levels.}
Algorithmically, extending the current two-level construction to a fully multilevel hierarchy is nontrivial. The upper-level system
\[
A_c\delta z=r_c
\]
is a Galerkin algebraic system, and different coarse factor-graph realizations may induce the same matrix equation. Thus a multilevel GBP construction must decide not only the dimension of each coarse variable, but also how unary and binary factors should be assigned among abstracted coarse nodes. A natural attempt is to assign one coarse variable to each fine group, with dimension equal to the retained local eigenbasis size, project within-group factors into unary coarse factors, and project cross-group factors into between-group coarse factors. In our preliminary experiments, however, GBP on this induced coarse factor graph typically required many iterations to solve the coarse system and was slower than a direct sparse solve at that level. This does not diminish the two-level result; rather, it highlights a compelling open direction: designing purely message-passing upper levels that retain the scalability and locality needed for fully distributed spatial-intelligence systems.

\paragraph{Broader applications.}
Finally, the abstraction-and-recovery principle behind H-GBP is not specific to pose graphs or bundle adjustment's cameras. Many sparse factor-graph problems have natural intermediate structures that could define useful coarse abstractions. In multi-robot SLAM and navigation, each robot's local subgraph could form a group, with higher-level messages coordinating inter-robot constraints. In visual odometry, pixels, patches, or object-level regions could be grouped into progressively coarser variables to accelerate GBP-based inference. Exploring such applications may reveal when hierarchical message passing can provide both distributed computation and strong global correction in broader sparse graphical models.

\end{document}